\documentclass[accepted]{uai2026} %
\usepackage[american]{babel}
\usepackage{natbib} 
\renewcommand{\bibsection}{\subsubsection*{References}}
\usepackage{mathtools} 
\usepackage{booktabs} 
\usepackage{tikz} 

\title{Capacity and Redundancy Trade-offs in Multi-Task Learning}
\usepackage{enumitem}
\usepackage{amsmath,amssymb,amsthm,mathtools}
\usepackage{hyperref}
\usepackage[capitalize,noabbrev]{cleveref}
\usepackage{algorithm}
\usepackage{algpseudocode} 
\usepackage{wrapfig}
\usepackage{subcaption}
\usepackage{graphicx,booktabs}
\usepackage{tikz}
\usetikzlibrary{shadows}
\usetikzlibrary{arrows.meta, positioning, calc, shapes.geometric, decorations.pathmorphing, backgrounds, fit}
\usepackage[style=base]{caption}
\DeclareMathOperator{\TC}{TC}

\newcommand{\Dsh}{\Delta_{\mathrm{sh}}}
\newcommand{\Dcl}[1]{\Delta_{#1}}
\newcommand{\TCbetween}{\TC_{\mathrm{between}}}

\newcommand{\tasks}{[T]}

\usepackage{pgfplots}
\pgfplotsset{compat=1.17}
\definecolor{grailBlue}{RGB}{50, 100, 160}
\definecolor{grailRed}{RGB}{200, 60, 60}
\definecolor{grailTeal}{RGB}{0, 130, 130}
\definecolor{grailGray}{RGB}{80, 80, 90}
\definecolor{rebuttalBlue}{RGB}{0, 76, 153}
\newif\ifshowrebuttalchanges
\showrebuttalchangestrue

\hypersetup{colorlinks, citecolor=blue, linkcolor=red, urlcolor=blue}
\theoremstyle{plain}
\newtheorem{theorem}{Theorem}[section]
\newtheorem{proposition}[theorem]{Proposition}
\newtheorem{lemma}[theorem]{Lemma}
\newtheorem{corollary}[theorem]{Corollary}
\theoremstyle{definition}
\newtheorem{definition}[theorem]{Definition}

\theoremstyle{remark}

\author[1]{\href{mailto:<asif.khan@hms.harvard.edu>?Subject=Your UAI 2026 paper}{Asif Khan}{}}
\affil[1]{%
Harvard Medical School, Boston, USA
}
\begin{document}
\maketitle
\begin{abstract}
In multi-task learning (MTL) negative transfer is often considered as an optimization artifact, but it can also be viewed as a consequence of limited shared capacity and weak task redundancy. We investigate this effect through a Capacity--Redundancy (CR) identity that decomposes the sum of per-task predictive informations into joint predictive information that includes label redundancy defined via total correlation (TC), and a residual coupling term that quantifies interference left unresolved by the shared representation. Additionally, we show two key results: (i) a clustering-gap decomposition that gives a necessary and sufficient condition for clustered sharing to outperform global sharing, and (ii) a gradient--TC bridge in a Gaussian multi-task model that formally justifies gradient cosine similarity as a proxy for redundancy ordering. Empirically, we estimate the residual coupling $\Delta$ from validation residual correlations, showing that clustered LoRA substantially reduces $\widehat{\Delta}$, outperforms size-matched random partitions, and results in statistically significant gains with multi-seed confidence intervals.
\end{abstract}

\section{Introduction}
\label{sec:intro}
A single machine learning model trained to optimize multiple prediction tasks simultaneously is useful for improving efficiency and generalization performance~\citep{caruana1997multitask,ruder2017overview}. In a typical setting, MTL is implemented through a shared encoder with task-specific output heads, which allows information to be transferred across tasks while retaining task specialization. Such a training approach has been successfully applied in several areas including computer vision (CV)~\citep{zhang2014facial,misra2016cross}, natural language processing (NLP)~\citep{collobert2008unified}, and speech recognition \citep{deng2013new}. However, despite these successes, there is a key challenge of negative transfer (when training on additional tasks degrades rather than improving performance)~\citep{wang2019characterizing}.

Recent developments in pre-trained large language models (LLMs) have renewed interest in MTL in the context of parameter-efficient fine-tuning (PEFT). Methods such as adapters \citep{houlsby2019parameter}, prefix-tuning \citep{li2021prefix,ding2023parameter}, and low-rank adaptation (LoRA) \citep{hu2022lora} make it possible to adapt massive models to new tasks with only a small number of additional parameters. A design choice here is to decide when multiple tasks should share the same lightweight module encouraging cross-task transfer, or should each task be assigned its own module to avoid interference? While previous work have explored both ends as well as various hybrids~\citep{he2021towards,karimi2021compacter,mahabadi2021parameter}, our focus is on a theoretical framework that can be used to decide when to share and when to specialize model parameters.

Negative transfer occurs when unrelated tasks share a low-rank subspace, whose  gradient signals may conflict and cause degradation of performance across all tasks~\citep{yu2020gradient,chen2018gradnorm}. Giving each task its own private features can avoid this interference but can reduce efficiency. Hybrid approaches such as task clustering, mixture-of-adapters, or routing via task embeddings~\citep{gururangan2021demix,mudrakarta2018k}, attempt to balance these trade-offs but lack a unifying theoretical justification.  

In this paper, we treat a shared encoder in MTL as a finite-capacity channel, given the budget $I(X;Z_s)\le C_s$ on the shared latent $Z_s$. We then prove a CR inequality that shows the total predictive information a single shared latent can provide across tasks is bounded by its capacity plus the label redundancy. Thus, correlations let the same bits be reused across tasks, while weakly related tasks force competition for capacity and make negative transfer unavoidable. We extend the bound to shared--private representations with per-task budgets and derive conditional variants together with a Bayes-error lower bound that formalize when performance trade-offs cannot be avoided. Finally, we model LoRA as a capacity-constrained channel and show that the CR perspective explains how similarity-based sharing helps when task redundancy is high, whereas low-redundancy task sets require growing effective adapter capacity (rank) or private routes to prevent interference.

In summary, we develop operational limits of sharing features and use it to prescribe rules for explicit capacity budgets. Our key contributions are: (i) We isolate the slack term $\Delta=\TC(Y^{(1:T)}\mid Z_s)$ and interpret it as interference left after conditioning on the shared representation. (ii) We derive an exact decomposition of the gain of clustered sharing over global sharing into interference reduction minus redundancy loss (\cref{thm:clustering_gap}), which provides a necessary and sufficient condition for when clustering should help. (iii) In a Gaussian multi-task model, we prove that gradient cosine similarity preserves the sign and (under matched noise) the ordering of label correlations, and therefore induces the same clusters as total-correlation-based redundancy (\cref{thm:grad_tc_bridge}). (iv) We show the CR bound is achieved with equality when capacity spans the task subspace and quantify the exact gap under rank constraints via $\TC(Y^{(1:T)}\mid Z_s)$. These results make the intuitive capacity-versus-redundancy tradeoff quantitative: \cref{thm:clustering_gap} gives an if-and-only-if sharing rule, the Gaussian specialization identifies when the CR bound is tight, and the LoRA rank bound connects adapter capacity to the effective task subspace dimension.

\section{Related Work}
\label{sec:related}
\textbf{Multi-task learning.}
Early work on MTL utilize a shared-backbone as well as task-specific heads to show consistent gain in performance when tasks are related~\citep{caruana1997multitask}. \citet{collobert2008unified} propose a unified architecture jointly training several NLP sequence tasks. In CV, multi-task feature sharing has been used to exploit structural relatedness between tasks. For example, \citet{zhang2014facial} showed improvement in facial landmark detection when jointly trained with auxiliary tasks such as head pose estimation, gender, and smile classification, since these signals share low-level visual features. Similarly, \citet{misra2016cross} introduced cross-stitch networks, which explicitly learns linear combination of feature maps across tasks, thus allowing the model to adaptively decide how much representation should be shared versus kept task-specific. Beyond deep nets, a large body of regularization approaches formalizes parameter sharing via matrix norms or task relationships \citep{evgeniou2004regularized,ando2005framework,argyriou2008convex}. We refer readers to \citet{ruder2017overview} for a comprehensive survey of MTL.

\textbf{Parameter-efficient fine-tuning and adapters.} Adapter layers are used as a light-weight alternative to full fine-tuning by inserting small bottleneck modules into frozen backbones \citep{houlsby2019parameter}. Subsequent PEFT methods include prefix/prompt tuning \citep{li2021prefix}. \citet{ding2023parameter}, and \citet{he2021towards} provide a much more detailed overviews of PEFT for LLMs. Multi-task PEFT variants allocate a shared adapter or per-task adapters, or combine both through routing, hypernetworks, or mixture strategies \citep{mahabadi2021parameter,karimi2021compacter,mudrakarta2018k,gururangan2021demix}. While empirical heuristics guide which modules to share, principled selection criteria remains limited.

\textbf{Empirical analyses of task relatedness and negative transfer.}
\citet{standley2020tasks} investigate which tasks benefit from being trained together and which cause conflicts, while \citet{zamir2018taskonomy} chart the transferability between tasks through a large-scale task graph. Gradient-based analyses of task conflict and mitigation are widely used in practice~\citep{yu2020gradient,chen2018gradnorm,chai2023getmtl}. GradNorm \citep{chen2018gradnorm}, PCGrad \citep{yu2020gradient} treat interference as a gradient conflict to be resolved step-by-step. This works well in practice but it offers no guarantee that interference can be resolved for a given capacity. Broader studies characterize and attempt to avoid negative transfer \citep{wang2019characterizing}. These methods typically operate at the level of optimization dynamics rather than at the level of distributional limits.

\textbf{Theoretical generalization and representation sharing.}
\citet{baxter2000model} formalized task families and bias learning; \citet{evgeniou2004regularized} analyzed kernelized regularization for MTL; and \citet{maurer2016benefit} show the benefits of learning a shared representation with task-averaged Rademacher complexity bounds. Several other works explicitly partition representations or allocate small task-specific heads to mitigate conflicts \citep{newell2019featurepartition,wang2020smalltowers}. Some other related approaches analyze multi-task trade-offs through multi-objective optimization and controllable Pareto frontiers to show explicit preference control during inference or training \citep{lin2020controllable,momma2022pareto}. \citet{maurer2016benefit} propose a bound on generalization error and show that sharing helps but their bound loosens as tasks become dissimilar without explaining where the capacity goes. In contrast, our focus is on a distribution-level converse on the achievable predictive information through a shared latent which is independent of any estimator. Our work is complementary to adaptive-transfer and selective-sharing mechanism using clustered task structure \citep{kang2011learning,jacob2008clustered}. \citet{denevi2022conditional}, \citet{chua2021fine}, and \citet{tian2025similar} analyze representation or meta-learning gains through estimator-dependent rates and robustness to similar but nonidentical representations, and \citet{hanneke2022no} show that multitask benefit cannot be guaranteed without structural assumptions. CR differs by giving a distribution-level converse on predictive information through a fixed-capacity latent, with the structural condition expressed as capacity versus redundancy.

\textbf{Multi-task information bottleneck.} Multi-task variational information bottleneck (IB)  $I(Z;Y)-\beta I(Z;X)$  formulations optimize multi-task IB objectives ($\sum_t I(Z;Y^{(t)})$) directly \citep{qian2020mtvib} which offers algorithmic objectives but not converses \citep{tishby2000information,alemi2016deep}. Classical results used in our analysis include the independence bound on entropy~\citep{cover1999elements} and its conditional form and the data processing inequality~\citep{cover1999elements}, total correlation \citep{watanabe1960information} as a redundancy measure, and subset entropy inequalities \citep{sun1975linear,madiman2010information}. Our results characterize when sharing is fundamentally limited and when structured sharing (clustering/private capacity) is provably beneficial, independent of a particular training algorithm.

\begin{figure*}[t]
        \centering
        \includegraphics[width=\linewidth]{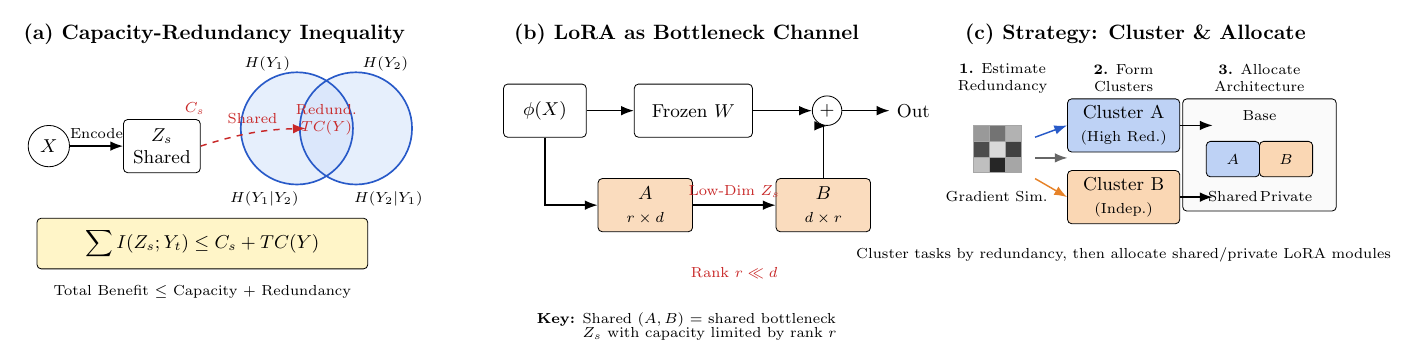}
        \caption{\textbf{Capacity-Redundancy Tradeoff in Multi-Task Low-Rank Adaptation.}
\textbf{(a)} shows task entropies $H(Y_1)$, $H(Y_2)$ with overlap $TC(Y) = I(Y_1; Y_2)$ representing redundancy. The shared bottleneck $Z_s$ (gray) has capacity $C_s$. Blue arrows indicate task-specific information $I(Z_s; Y_t)$.
\textbf{(b)} rank-$r$ matrices $(A,B)$ add to frozen weights $W$, where rank $r \ll d$ limits capacity $C_s$. Sharing $(A,B)$ across tasks creates the shared bottleneck.
\textbf{(c)} estimate redundancy via gradient similarity (step 1), cluster tasks by redundancy (step 2), then allocate shared LoRA modules to redundant clusters and private modules to independent tasks (step 3), optimizing the CR bound.}
\label{fig:placeholder}
\end{figure*}
\section{Capacity--Redundancy Inequality}
\label{sec:cr}
Let $T$ be the number of tasks with labels $Y^{(1)},\dots,Y^{(T)}$ and input features $X$. A shared representation $Z_s=f(X)$ is used across all tasks. Throughout this section, we assume the standard representation setting in which $Y^{(1:T)} \;-\; X \;-\; Z_s$ forms a Markov chain (equivalently, $I(Z_s;Y^{(1:T)}\mid X)=0$), and we impose the shared capacity budget  $I(Z_s;X)\;\le\; C_s.$ We use total correlation as a measure of label redundancy, $TC(Y^{(1:T)}) \;=\; \sum_{t=1}^T H\!\left(Y^{(t)}\right) \;-\; H\!\left(Y^{(1:T)}\right) \;\ge\; 0 .$ $TC(Y^{(1:T)})=0$ iff the labels are mutually independent. The sum of per-task information carried by a shared bottleneck is always bracketed by the joint information plus redundancy. The main idea requires only the chain rule and data processing, so we focus on implication of bound and for proofs/preliminaries we refer readers to Appendix~\ref{sec:supp_notation}. 

\begin{lemma}[Chain-rule sandwich]\label{lem:sandwich}
For random variables $(Z,Y^{(1)},\dots,Y^{(T)})$,
\begin{align}
I\!\left(Z;Y^{(1:T)}\right)
\;\le\;
\sum_{t=1}^T I\!\left(Z;Y^{(t)}\right)
\;\le\; \nonumber\\
I\!\left(Z;Y^{(1:T)}\right) + TC\!\left(Y^{(1:T)}\right).
\end{align}
\label{eq:sandwich}
\end{lemma}
Next, we combine the sandwich with the Markov chain and the capacity budget to define CR inequality.

\begin{theorem}[CR inequality]\label{thm:cr}
Under the Markov chain and a fixed capacity budget the following bound holds,
\begin{equation}
\sum_{t=1}^T I\!\left(Z_s;Y^{(t)}\right)
\;\le\;
C_s + TC\!\left(Y^{(1:T)}\right).
\label{eq:cr}
\end{equation}
and the canonical sandwich Lemma \ref{eq:sandwich} holds when $Z=Z_s$,
\begin{align}
I\!\left(Z_s;Y^{(1:T)}\right)
\;\le\;
\sum_{t=1}^T I\!\left(Z_s;Y^{(t)}\right)
\;\le\;\nonumber\\
I\!\left(Z_s;Y^{(1:T)}\right)+TC\!\left(Y^{(1:T)}\right).
\label{eq:cr_sandwich}
\end{align}
\end{theorem}
The upper bound in \cref{eq:cr} is induced by an exact identity,
\begin{align*}
    \sum_{t=1}^T I\!\left(Z_s;Y^{(t)}\right) =
I\!\left(Z_s;Y^{(1:T)}\right) + \\\quad\underbrace{\Big(TC(Y^{(1:T)})-TC(Y^{(1:T)}\mid Z_s)\Big)}_{\text{dependence explained by }Z_s}.
\end{align*}
which is equivalent to $\sum_{t=1}^T I\!\left(Z_s;Y^{(t)}\right) = I\!\left(Z_s;Y^{(1:T)}\right)+TC(Y^{(1:T)})-\Delta$, where $\Delta = TC(Y^{(1:T)}\mid Z_s)\ge 0$. 

Here, $\Delta$ is the residual multi-task coupling that remains after observing $Z_s$, $\Delta=0$ iff the labels become conditionally independent given $Z_s$. When $\Delta$ is large, the right inequality in \cref{eq:cr_sandwich} can be loose, meaning the shared representation has not factorized the task joint distribution, and that additional shared or task-specific capacity (\cref{sec:shared_private}) is required.

The tightness of \cref{eq:cr} depends on conditional independence of tasks given $Z_s$  i.e. $\Delta=TC(Y^{(1:T)}\mid Z_s)\approx 0$, and capacity saturation $I(Z_s;Y^{(1:T)})\approx I(Z_s;X)\approx C_s$. In a linear--Gaussian setting dependence is explained by $Z_s$ spanning the label-relevant subspace so that cross-task coupling disappears after conditioning, while saturation corresponds to operating at the channel capacity of the bottleneck.

\section{Consequences and Impossibility Results}
\label{sec:consequences}
The CR inequality (\cref{thm:cr}) has a sharp consequence: without redundancy, adding tasks to a fixed-capacity bottleneck must degrade per-task performance as the number of tasks grows unless capacity scales. It further implies that the total per-task information extractable from a shared bottleneck cannot exceed shared capacity plus label redundancy $\sum_{t=1}^T I(Z_s;Y^{(t)}) \;\le\; C_s + TC(Y^{(1:T)}).$

\begin{corollary}[Average per-task information]\label{cor:avg_mi}
Under the assumptions of ~\cref{thm:cr},
\begin{equation}
\frac{1}{T}\sum_{t=1}^T I(Z_s;Y^{(t)})
\;\le\;
\frac{C_s}{T} + \frac{TC(Y^{(1:T)})}{T}.
\label{eq:avg_mi}
\end{equation}
In particular, there exists at least one task $t^\star$ such that $I(Z_s;Y^{(t^\star)}) \;\le\; \frac{C_s + TC(Y^{(1:T)})}{T}.$
\end{corollary}

When tasks are weakly dependent, $TC(Y^{(1:T)})$ does not grow with $T$, and \cref{eq:avg_mi} forces the average per-task information to decay like $O(1/T)$ for fixed $C_s$. 

\begin{corollary}[Linear capacity is necessary for maintaining per-task signal]\label{cor:linear_capacity}
Let $TC(Y^{(1:T)}) \le \tau$ for all $T$, now for a uniform per-task information level
\(
I(Z_s;Y^{(t)}) \ge b
\)
for all $t\in[T]$, then necessarily $C_s \;\ge\; Tb - \tau.$ In particular, for independent tasks ($TC=0$), maintaining $I(Z_s;Y^{(t)})\ge b$ for all $t$ forces $C_s\ge Tb$.
\end{corollary}

Corollary \ref{cor:linear_capacity} tells us when negative transfer is unavoidable. If new tasks are added without increasing shared capacity then some tasks must lose predictive signal through the shared bottleneck, unless the new tasks are redundant with existing ones. 

\textbf{Log-loss lower bound.} For each task $t$, let $q_t(\cdot\mid Z_s)$ denote any predictor and consider the log-loss $\ell_t = -\log q_t(Y^{(t)}\mid Z_s)$. Then by the Gibbs inequality,
\begin{equation}
\mathbb{E}[\ell_t] \;\ge\; H(Y^{(t)}\mid Z_s)
\;=\;
H(Y^{(t)}) - I(Z_s;Y^{(t)}).
\label{eq:logloss_basic}
\end{equation}
\begin{theorem}[Total log-loss lower bound with residual coupling]\label{thm:logloss_sp}
Under the assumptions of~\cref{thm:cr}, for any collection of per-task predictors $\{q_t\}_{t=1}^T$,
\begin{align}
\sum_{t=1}^T \mathbb{E}[-\log q_t(Y^{(t)}\mid Z_s)]
\;\ge\;
\sum_{t=1}^T H(Y^{(t)}\mid Z_s)
\;=\;\nonumber\\
H(Y^{(1:T)}\mid Z_s) \;+\; \Delta,
\label{eq:logloss_lb_b}
\end{align}
where $\Delta = TC(Y^{(1:T)}\mid Z_s)\ge 0$. Moreover, using $I(Z_s;Y^{(1:T)})\le I(Z_s;X)\le C_s$,
\begin{equation}
\sum_{t=1}^T \mathbb{E}[-\log q_t(Y^{(t)}\mid Z_s)]
\;\ge\;
H(Y^{(1:T)}) - C_s + \Delta.
\label{eq:logloss_lb_cs}
\end{equation}
\end{theorem}
The lower bound in~\cref{eq:logloss_lb_b} implies that $\Delta>0$ forces extra aggregate log-loss for any set of decoupled per-task predictors $\{q_t(\cdot\mid Z_s)\}$, unless the shared features $Z_s$ makes tasks conditionally independent. This is why private features (\cref{sec:shared_private}) are useful as they can reduce $\Delta$ by explaining away residual coupling that a shared representation leaves behind.

Appendix~\ref{app:fano} gives a Fano-style error lower bound as a corollary, while constants can be loose it still provides a clean impossibility statement directly in terms of misclassification.

\textbf{When should tasks share?} If tasks are redundant (large $TC(Y^{(1:T)})$ or large explained dependence $TC(Y^{(1:T)})-TC(Y^{(1:T)}\mid Z_s)$), then sharing features is favorable. If tasks are heterogeneous (small redundancy) and $T$ grows then~\cref{cor:linear_capacity} implies that any globally shared bottleneck must either increase capacity roughly linearly in $T$ or accept degraded per-task signal. Using this observation, we suggest clustered sharing and private capacity allocation as an approach to task heterogeneity.

\textbf{Gradient cosine similarity as a redundancy proxy.} Gradient cosine similarity is widely used to measure task relatedness, but does it track the redundancy that CR predicts. We justify gradient cosine similarity as a proxy for redundancy ordering in a Gaussian model.

\begin{theorem}[Gradient--TC Bridge (Gaussian multi-task model)]
\label{thm:grad_tc_bridge}
{
Let $X\sim \mathcal{N}(0,\Sigma)$ and $Y^{(t)} = w_t^\top X + \varepsilon_t$ with independent $\varepsilon_t\sim \mathcal{N}(0,\sigma_t^2)$. We then define whitened task directions $\tilde{w}_t=\Sigma^{1/2}w_t$ and the label correlation $\rho_{ts} = \frac{w_t^\top \Sigma w_s}{\sqrt{(w_t^\top \Sigma w_t+\sigma_t^2)(w_s^\top \Sigma w_s+\sigma_s^2)}}.$ The natural-gradient, equivalently whitened-gradient, cosine similarity at initialization is
\[
G^{\mathrm{nat}}_{ts}
=
\frac{\tilde{w}_t^\top \tilde{w}_s}{\|\tilde{w}_t\|\,\|\tilde{w}_s\|}
=
\frac{w_t^\top \Sigma w_s}{\sqrt{(w_t^\top \Sigma w_t)(w_s^\top \Sigma w_s)}}.
\]
Then $\mathrm{sign}(\rho_{ts})=\mathrm{sign}(G^{\mathrm{nat}}_{ts})$ for all $t\neq s$, and under matched noise levels and equal signal powers $w_t^\top\Sigma w_t$ across tasks, the ordering of $|\rho_{ts}|$ across pairs is the same as the ordering of $|G^{\mathrm{nat}}_{ts}|$.

For the raw-gradient cosine in unwhitened coordinates,
\[
G^{\mathrm{raw}}_{ts}
=
\frac{w_t^\top \Sigma^2 w_s}{\|\Sigma w_t\|\,\|\Sigma w_s\|},
\]
the exact equivalence need not hold under anisotropic $\Sigma$; it coincides with the whitened-gradient form when inputs are whitened or isotropic, with the discrepancy controlled by the conditioning of $\Sigma$.

Since for Gaussian labels $\TC(Y^{(1:T)}) = -\tfrac12\log\det R$ where $R$ is the correlation matrix, any clustering rule that depends only on the ranking of pairwise similarities results in the same partition when applied to $G^{\mathrm{nat}}$ as when applied to $R$.
}
\end{theorem}
\section{Shared--Private Extension}
\label{sec:shared_private}
In practice, multi-task models often augment shared features with task-private features that can absorb heterogeneity and mitigate interference. A shared-only bound ignores private capacity that most practical architectures already use. We therefore add per task-budgets and derive a corresponding shared--private CR bound. Let $Z_s=f_s(X)$ denote shared features and for each task $t\in[T]$, let $Z_t=f_t(X)$ denote task-private features. The task-specific representation is $Z^{(t)} := (Z_s, Z_t).$ We assume the (test-time) Markov property $Y^{(1:T)} \;-\; X \;-\; (Z_s, Z_1,\dots,Z_T),$ and impose separated information budgets,
\begin{align}
I(Z_s;X)\le C_s,
\qquad
I(Z_t;X\mid Z_s)\le C_t \quad \forall t\in[T].
\label{eq:sp_budgets}
\end{align}
where $I(Z_t;X\mid Z_s)$ is a conditional budget that measures incremental private capacity beyond what the shared bottleneck
already covers. This results in an additive upper bound on the total information extracted from $X$ by the collection
$(Z_s,Z_1,\dots,Z_T)$, even when the private feature maps are statistically dependent through $X$.

\begin{theorem}[Shared--private capacity--redundancy bound]
\label{thm:sp_cr}
Under Markov chain and \cref{eq:sp_budgets},
\begin{equation}
\sum_{t=1}^T I\!\left(Z^{(t)};Y^{(t)}\right)
\;\le\;
C_s + \sum_{t=1}^T C_t \;+\; TC\!\left(Y^{(1:T)}\right).
\label{eq:sp_cr}
\end{equation}
\end{theorem}
\cref{thm:sp_cr} mirrors the shared-only CR inequality, with the shared capacity $C_s$ replaced by a total budget $C_s+\sum_t C_t$. This shows that private features can reduce the residual task coupling left unexplained by a shared bottleneck. Shared-only slack $\Delta_s$ measures how task labels remain coupled after conditioning on the shared representation. When $\Delta_s$ is large, the shared representation has not factorized the joint task structure. In such a setting, global sharing is prone to interference unless $C_s$ grows.

We now use shared--private features to define the post-augmentation residual coupling $\Delta_{sp} = TC(Y^{(1:T)}\mid Z_{\mathrm{all}})$, where $Z_{\mathrm{all}}=(Z_s,Z_1,\dots,Z_T)$. Since conditioning cannot increase entropy, $TC(\cdot\mid\cdot)$ is monotone in the conditioning set, implying $\Delta_{sp} \;\le\; \Delta_s.$ Thus private features can only decrease residual coupling, and doing so is what makes per-task heads easier to fit without forcing the shared bottleneck to encode heterogeneous, task-specific details.

The bound in \cref{eq:sp_cr} is tight when the total representation saturates the information budget $I(Z_{\mathrm{all}};X)\approx C_s+\sum_t C_t$, and the residual dependence is largely removed, $\Delta_{sp}\approx 0$ (labels become approximately conditionally independent given $Z_{\mathrm{all}}$). In contrast, if $\Delta_s$ remains large at fixed $C_s$, then increasing private budgets $\{C_t\}$ is the only way within this framework to reduce $\Delta_{sp}$ and avoid the shared-only obstruction.

The shared--private inequality in \cref{eq:sp_cr} suggests to share capacity within groups of redundant tasks, and use private capacity to address residual coupling. Later in experiments we validate this by estimating task similarity and comparing global sharing, clustered sharing, and private adapters under matched total rank
budgets.

\subsection{Clustering as structured partial sharing}
\label{sec:clustering_gap}
MTL approaches already cluster tasks and share within clusters. The question is whether this provably helps and by how much? The basic CR bound alone cannot certify when clustering improves over global sharing. We now provide a sharp gain decomposition.
\begin{definition}[Inter-cluster total correlation]
\label{def:tc_between}
Let $\mathcal{P}=\{S_1,\dots,S_K\}$ be a partition of $\tasks$ into disjoint clusters. We define $\TCbetween(\mathcal{P})
\;=\;
\TC\!\left(Y^{(1:T)}\right)\;-\;\sum_{k=1}^K \TC\!\left(Y^{(S_k)}\right),$ as the redundancy that crosses cluster boundaries.
\end{definition}

\begin{theorem}[Clustering gap decomposition]
\label{thm:clustering_gap}
Assume the test-time Markov property $Y^{(1:T)}-X-\{Z_s, Z_{S_1},\dots,Z_{S_K}\}$. Let $Z_s^\star$ be a capacity-optimal globally shared encoder under budget $I(Z_s;X)\le C$.
Let $\{Z_{S_k}^\star\}_{k=1}^K$ be capacity-optimal per-cluster encoders with budgets $I(Z_{S_k};X)\le C_k$ and $\sum_k C_k=C$. We define residual couplings as $\Dsh = \TC\!\left(Y^{(1:T)}\mid Z_s^\star\right),
\Dcl{k} = \TC\!\left(Y^{(S_k)}\mid Z_{S_k}^\star\right).$ Then the sum-rate gain of clustering over global sharing decomposes as,
\begin{align*}
\sum_{t=1}^T I\!\left(Z_{S_{k(t)}}^\star; Y^{(t)}\right)
-
\sum_{t=1}^T I\!\left(Z_s^\star; Y^{(t)}\right)
\;=\;\\
\underbrace{\Big(\Dsh-\sum_{k=1}^K \Dcl{k}\Big)}_{\text{interference reduction }\ge 0}
\;-\;
\underbrace{\TCbetween(\mathcal{P})}_{\text{redundancy loss }\ge 0}.
\end{align*}
Thus clustered sharing is strictly beneficial if and only if the interference reduction exceeds the redundancy loss.
\end{theorem}
We can use~\cref{thm:clustering_gap} to choose partitions that keep strongly redundant tasks together (small $\TCbetween$) while separating antagonistic tasks to reduce residual coupling (large $\Dsh-\sum_k \Dcl{k}$).~\citet{fifty2021efficiently} search over task groupings combinatorially,~\cref{thm:clustering_gap} tells us why a particular grouping wins.
Here, the interference-reduction term $\Dsh-\sum_k \Dcl{k}$ measures how much residual cross-task coupling is removed when weakly related tasks no longer compete for one shared representation. The redundancy-loss term $\TCbetween(\mathcal{P})$ measures useful signal that crosses cluster boundaries and is no longer shared by a single global representation. Thus clustering is most beneficial when the task graph has clear communities with high intra-cluster redundancy and low inter-cluster redundancy.
\section{Subsets and Side Information}
\label{sec:refinements}
The global CR bound treats all T tasks as one block, which is wasteful when tasks cluster. In such cases, aggregate CR bound over all tasks may be too coarse to reflect finer-grained task relationships. Subset bounds exploit this structure.

\textbf{Subset bounds.} Let $S\subseteq[T]$ be any subset of tasks and denote $Y^S = \{Y^{(t)}:t\in S\}$. Then the subset total correlation is defined as $TC(Y^S) \;=\; \sum_{t\in S} H(Y^{(t)}) \;-\; H(Y^S).$ Applying Lemma~\ref{lem:sandwich} and ~\cref{thm:cr} to the subcollection results in an immediate tightening.

\begin{corollary}[Subset CR inequality]\label{cor:subset_cr}
Under the assumptions of~\cref{thm:cr}, for any $S\subseteq[T]$,
\begin{equation}
\sum_{t\in S} I(Z_s;Y^{(t)}) \;\le\; C_s + TC(Y^S).
\label{eq:subset_cr}
\end{equation}
Moreover, the canonical sandwich holds on $S$,
\begin{equation}
I(Z_s;Y^S)
\;\le\;
\sum_{t\in S} I(Z_s;Y^{(t)})
\;\le\;
I(Z_s;Y^S) + TC(Y^S).
\label{eq:subset_sandwich}
\end{equation}
\end{corollary}

The subset bound is usually tighter than~\cref{eq:cr} because redundancy is not uniform, tasks often form clusters with high intra-cluster dependence but low inter-cluster dependence. This refinement suggests that shared capacity should be allocated relative to $TC(Y^S)$ within each cluster, with respect to the global quantity $TC(Y^{(1:T)})$.

\textbf{Side information and conditional CR.} Let $W$ denote side information available when constructing or analyzing the representation. The conditional total correlation is defined as $TC(Y^{(1:T)}\mid W)
\;=\;
\sum_{t=1}^T H(Y^{(t)}\mid W) 
\;-\; H(Y^{(1:T)}\mid W)
\;\ge\; 0.$

\begin{theorem}[Conditional CR inequality]\label{thm:cond_cr}
Under a conditional Markov property $Y^{(1:T)} \;-\; X \;-\; Z_s \quad \text{given } W$, and a shared capacity budget  $I(Z_s;X\mid W)\le C_s(W)$,
\begin{equation}
\sum_{t=1}^T I(Z_s;Y^{(t)}\mid W)
\;\le\;
C_s(W) + TC(Y^{(1:T)}\mid W).
\label{eq:cond_cr}
\end{equation}
More generally, for any subset $S\subseteq[T]$,
\begin{equation}
\sum_{t\in S} I(Z_s;Y^{(t)}\mid W)
\;\le\;
C_s(W) + TC(Y^S\mid W).
\label{eq:cond_subset_cr}
\end{equation}
\end{theorem}

\textbf{Residual coupling with side information.}
The exact decomposition also holds conditionally $\sum_{t=1}^T I(Z_s;Y^{(t)}\mid W) = I(Z_s;Y^{(1:T)}\mid W) + \Big(TC(Y^{(1:T)}\mid W)-TC(Y^{(1:T)}\mid Z_s,W)\Big)$ and the conditional slack $\Delta(W) = TC(Y^{(1:T)}\mid Z_s,W)\ge 0$ again measures residual task coupling after observing $Z_s$, now at a fixed value of side information $W$. During training, the representation parameters depend on labels, so the unconditional Markov chain $Y^{(1:T)}-X-Z_s$ need not hold. However, at evaluation time weights are fixed making it possible to use conditional Markov property, DPI and CR arguments at test time. This is crucial as it separates learning dynamics from the representational limit that are the main focus of this work.

If $W$ includes task identity or context, then $TC(Y^{(1:T)}\mid W)$ can be much smaller than $TC(Y^{(1:T)})$, since conditioning can decouple labels across tasks. In such a setting, using conditional adapters or MoE routing increase effective performance by reducing conditional redundancy terms or allowing $C_s(W)$ to vary with $W$. The conditional CR bound \cref{eq:cond_cr} serves as a unified lens for these designs. It also covers task-dependent input distributions by taking $W$ to include task identity or context.

\textbf{Implications for clustered sharing.} Subset and conditional refinements together suggest to first estimate task similarity (as a proxy for dependence), cluster tasks into subsets $\{S_k\}$, and allocate shared capacity per cluster. Formally, applying \cref{eq:subset_cr} within each cluster gives $\sum_{t\in S_k} I(Z_s^{(k)};Y^{(t)}) \;\le\; C_s^{(k)} + TC(Y^{S_k})$ so at fixed total budget $\sum_k C_s^{(k)}$, clustered sharing targets the local redundancy structure rather than paying for global heterogeneity. We refer readers to Appendix~\cref{sec:gaussian} for a Gaussian model where both shared capacity and redundancy can be computed in closed form.
\section{LoRA Fine-tuning}
\label{sec:lora}
Next, we introduce a capacity proxy to link adapter rank to an effective shared budget $C_s$ that results in a testable predictions about when globally shared low-rank updates saturate and when clustered/private updates are necessary.

\textbf{LoRA as a low-rank information channel.} Consider a frozen backbone with a hidden representation $\phi(X)\in\mathbb{R}^d$. A LoRA update replaces a linear map $W\in\mathbb{R}^{o\times d}$ by
\begin{equation}
W \;\mapsto\; W + \Delta W,
\quad
\Delta W = B A,
\quad
A\in\mathbb{R}^{r\times d},\; B\in\mathbb{R}^{o\times r},
\label{eq:lora_update}
\end{equation}
where $r\ll \min\{o,d\}$ is the LoRA rank. We view LoRA branch as a feature channel that compresses the input into a low-dimensional signal $Z_s \;=\; A\,\phi(X) \in \mathbb{R}^r,$ which is then mixed into the output via $B$. In a multi task setup, this implies a structural constraint that tasks sharing the same LoRA factors share the same bottleneck $Z_s$. Thus, the shared representation capacity is controlled by the rank $r$ and the scale of $A$ relative to the feature distribution $\phi(X)$. This is consistent with the intrinsic-dimensionality perspective~\citep{aghajanyan-etal-2021-intrinsic}.

\textbf{CR predictions for LoRA design.} Let us consider $T$ tasks fine-tuned by sharing the same LoRA factors $(A,B)$ (and thus the same bottleneck $Z_s$). Applying~\cref{thm:cr} with shared budget $C_s \approx I(Z_s;\phi(X))$ the capacity can be upper bound via \cref{eq:lora_capacity_exact}. If tasks are heterogeneous so that $TC(Y^{(1:T)})$ is small, then CR prediction implies that the sum of task-relevant information is controlled primarily by $I(Z_s;\phi(X))$, which grows slowly with $r$. Consequently, performance under a single globally shared LoRA should saturate quickly as rank increases, and adding tasks without increasing rank induces interference (\cref{sec:consequences}).

\begin{table}
\centering
\begin{tabular}{lccc}
\toprule
 & \textbf{r=4} & \textbf{r=8} & \textbf{r=16} \\
\midrule
Shared LoRA      & 0.763 & 0.773 & 0.755 \\
Private LoRA     & 0.796 & 0.795 & 0.797 \\
Clustered LoRA   & 0.812 & 0.817 & 0.806 \\
\bottomrule
\end{tabular}
\caption{\textbf{Mean validation accuracy} on SST2/MRPC/RTE for budget-valid total ranks. Clustered sharing consistently outperforms global sharing under matched rank budgets.}
\label{tab:glue_rank}
\end{table}
To obtain a per-task information level $I(Z_s;Y^{(t)})\gtrsim b$ across many weakly dependent tasks, Corollary~\ref{cor:linear_capacity} tells us that the total shared information $I(Z_s;\phi(X))$ must scale approximately like $Tb$ (modulo redundancy). Under our rank-controlled bound, LoRA rank must grow roughly linearly in the number of distinct task directions or clusters. If tasks form clusters $\{S_k\}$ with high within-cluster dependence but low across-cluster dependence, then we can use the subset CR bound (\cref{cor:subset_cr}), sharing LoRA within each cluster and allocating rank per cluster,
\[
\sum_{t\in S_k} I(Z_s^{(k)};Y^{(t)}) \;\le\; C_s^{(k)} + TC(Y^{S_k}),
\quad
\sum_k C_s^{(k)}\le C_{\mathrm{tot}}.
\]

\textbf{Estimating residual coupling \texorpdfstring{$\Delta$}{Delta} from validation residuals.} We use validation data $\{(x_i,y_i^{(1:T)})\}_{i=1}^n$ to compute per-task residuals
$e_{t}^{(i)} = y_{t}^{(i)} - \hat y_{t}^{(i)}$. Let $\widehat{R}_e$ be the sample correlation matrix of $(e_1,\dots,e_T)$. We approximate $\widehat{\Delta}
\;=\;
-\frac12 \log\det(\widehat{R}_e+\epsilon I)
\;\approx\;
\TC\!\left(Y^{(1:T)}\mid Z_s\right),$ and report it across different settings. In NLP experiments we use $\epsilon=10^{-5}$; the synthetic sensitivity sweep below shows this value is in the near-unbiased regime for our validation sample sizes, while larger $\epsilon$ persistently underestimates the conditional TC.
\section{Empirical Validation}
\subsection{Validation on synthetic data.}
\label{sec:synthetic}
\textbf{CR inequality for Gaussian case.} We first verify the CR inequality in a linear--Gaussian multitask model where all terms are available in closed form, and show that the bound is tight once the representation spans the task subspace. In particular, when the encoder row-space contains the span of task directions, the residual coupling term vanishes, while under rank constraints the CR slack equals the conditional total correlation $\TC(Y^{1:T}\mid Z)$. Full derivations and controlled sweeps across redundancy, encoder alignment, and rank are provided in Appendix~\ref{sec:supp-linear-gaussian}.

\textbf{Validation for $\widehat{\Delta}$.} We decouple estimator behavior from representation learning using a synthetic Gaussian multitask model where the target conditional TC quantity is known in closed form. We then estimate $\widehat{\Delta}$ from $N$ i.i.d.\ samples and study both absolute error and signed bias across ridge parameters $\varepsilon$. Figure~\ref{fig:gaussian} shows that small $\varepsilon$ in a near-unbiased estimates as $N$ increases, while larger $\varepsilon$ induces a persistent negative bias that does not vanish with sample size. At the largest sample size ($N=16384$), Table~\ref{tab:gaussian_eps} quantifies this effect $\varepsilon\in\{10^{-6},10^{-5}\}$ achieves small error, whereas $\varepsilon\ge 10^{-4}$ results in an underestimation. This shows that the $\varepsilon$ acts as a regularization term for a log-determinant inverse problem, and using small $\varepsilon$ in downstream evaluations. This sensitivity analysis directly motivates the choice $\epsilon=10^{-5}$ used for the LoRA residual-coupling estimates.

\begin{figure}
    \centering
    \includegraphics[width=\linewidth]{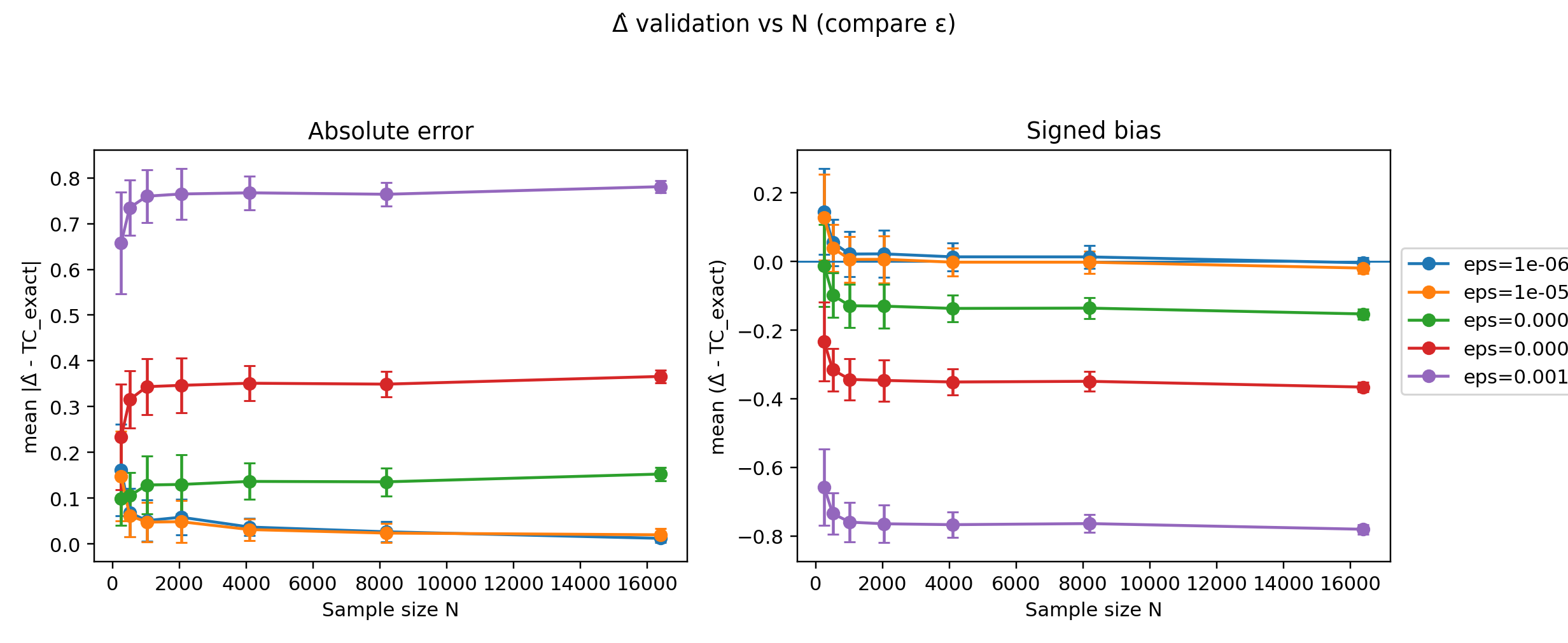}
  \caption{Mean absolute error $|\widehat{\Delta}-\mathrm{TC}_{\text{exact}}|$ vs sample size $N$ on the left. Signed bias $\widehat{\Delta}-\mathrm{TC}_{\text{exact}}$ vs $N$ on the right.
  Small $\varepsilon$ converges toward the ground truth as $N$ increases, while larger $\varepsilon$ results in a persistent negative bias.}
  \label{fig:gaussian}
\end{figure}

\begin{table}[t]
\centering
\caption{Gaussian setting: error of $\hat\Delta$ at the largest sample size ($N=16384$, 10 seeds).}
\label{tab:gaussian_eps}
\begin{tabular}{lrrrr}
\toprule
$\varepsilon$ & \shortstack{mean $|\hat\Delta-\mathrm{TC}|$} & std & \shortstack{mean\\signed} & std \\
\midrule
1e-06 & 0.0123 & 0.0089 & -0.0037 & 0.0153 \\
1e-05 & 0.0201 & 0.0138 & -0.0191 & 0.0152 \\
1e-04 & 0.1528 & 0.0148 & -0.1528 & 0.0148 \\
3e-04 & 0.3660 & 0.0143 & -0.3660 & 0.0143 \\
1e-03 & 0.7807 & 0.0138 & -0.7807 & 0.0138 \\
\bottomrule
\end{tabular}
\end{table}

\subsection{CR aligned PEFT.}
We now evaluate applicability of CR in a real setting by training multi-task PEFT with a frozen \texttt{bert-base-uncased} encoder and LoRA adapters applied to attention query and value projections.

\paragraph{GoEmotions Dataset.} 
We use GoEmotions~\citep{demszky-etal-2020-goemotions} dataset which consists of shared inputs and multiple binary labels. We treat each label as a task sharing the same input text and select twelve most frequent GoEmotions labels (fixed across runs) and evaluate each task using AUROC on the validation split and report the macro-average AUROC across the chosen tasks.

We analyze two finetuning setting under a fixed LoRA rank budget $r_{\mathrm{tot}} \in \{4,8,12,16,24,32\}$: using one shared adapter across all tasks, and cluster tasks into K clusters based on similarity and use K adapters. For clustered routing, the clustering is done once using gradient-similarity probes (details in Appendix~\ref{sec:supp-setup}). We also track $\widehat{\Delta}$ as a diagnostic of residual dependence. Figure~\ref{fig:goemotions_auc} shows that similarity-based routing substantially improves performance under small budgets. From $r_{\mathrm{tot}}=8$ onward, clustered consistently outperforms shared in average AUCROC (Table~\ref{tab:goemotions_auc}). Moreover, the clustered routing results in a systematically lower $\widehat{\Delta}$ at matched rank budgets (Figure~\ref{fig:goemotions_delta}) which shows routing reduces residual interference in the learned predictor.

\begin{figure}[t]
    \centering
    \includegraphics[width=0.7\linewidth]{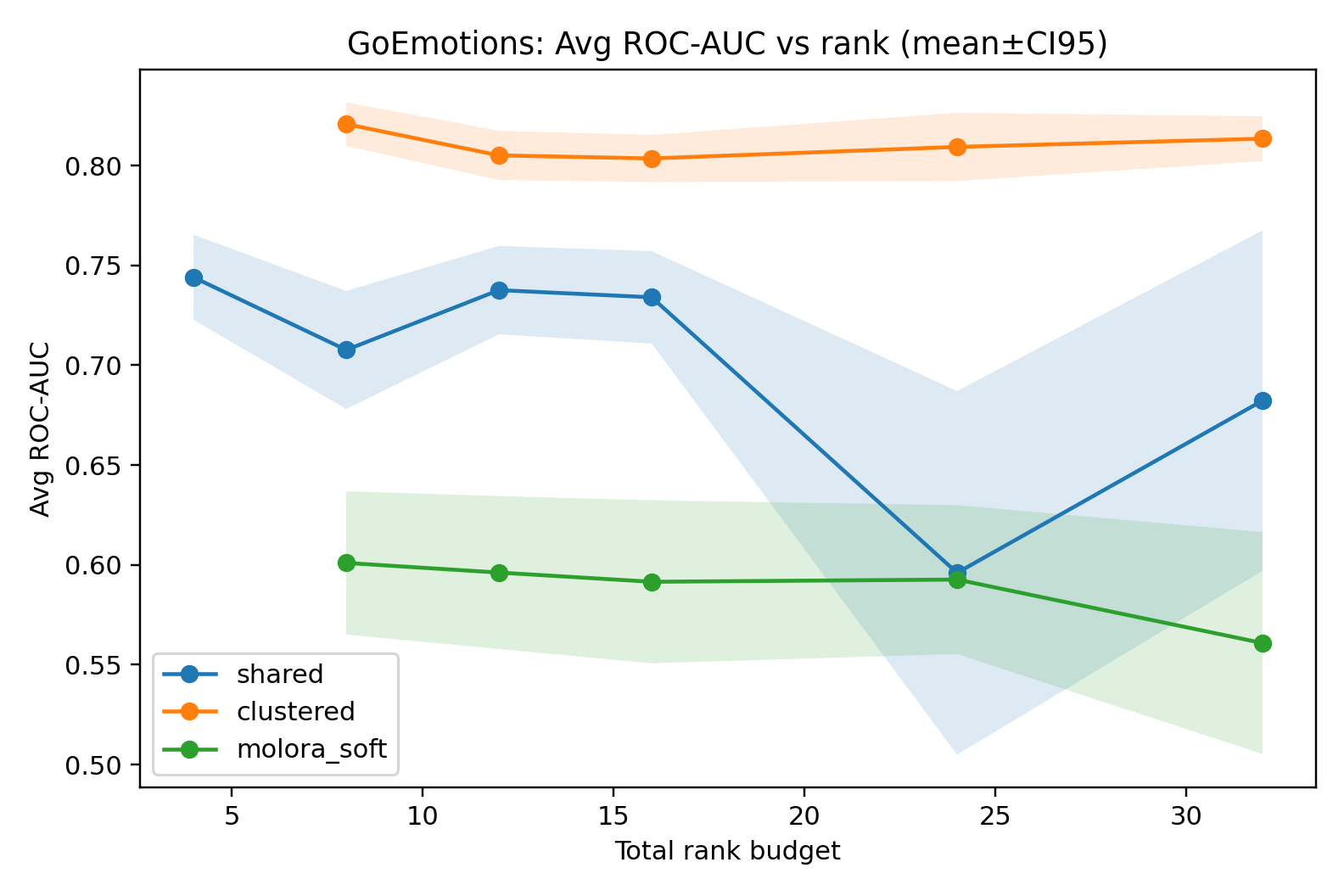}
    \caption{Macro-average AUROC on the validation split as a function of total LoRA rank budget $r_{\mathrm{tot}}$. Clustered routing consistently improves over naive sharing from $r_{\mathrm{tot}}=8$ onward. Shaded regions denote CI.}
    \label{fig:goemotions_auc}
\end{figure}

\begin{figure}[t]
    \centering
    \includegraphics[width=0.7\linewidth]{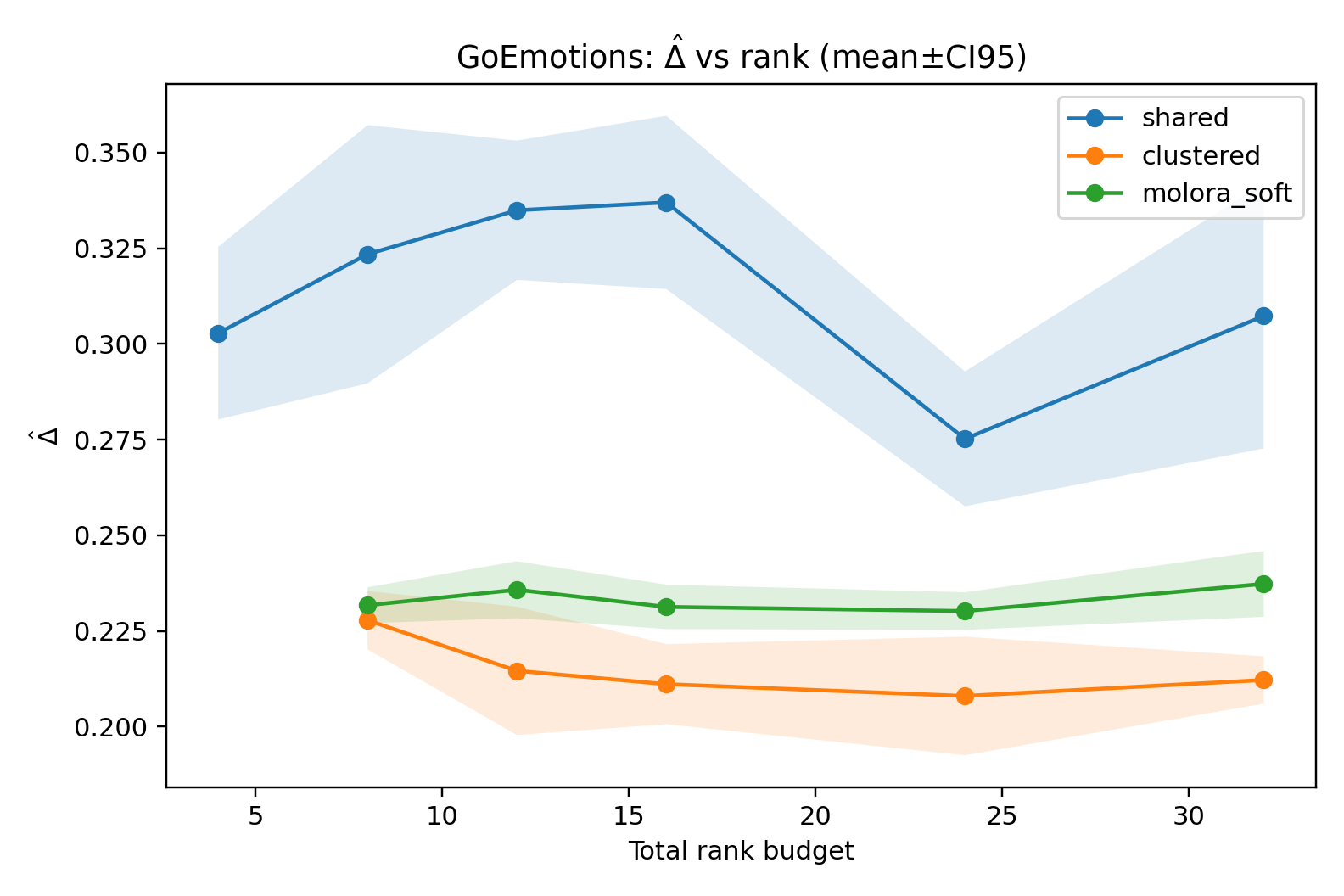}
  \caption{$\widehat{\Delta}$ computed from validation residuals vs total rank budget.
  Clustered routing results in a systematically smaller $\widehat{\Delta}$ than sharing at matched budgets (rank $\ge 8$), showing reduced residual dependence under similarity-guided finetuning.}
  \label{fig:goemotions_delta}
\end{figure}

\begin{table}[t]
\centering
\caption{Average AUROC (mean $\pm$ CI95) vs total rank budget.}
\label{tab:goemotions_auc}
\begin{tabular}{lccc}
\toprule
Rank & Shared & Clustered \\
\midrule
4 & 0.744 $\pm$ 0.021 & -- \\
8 & 0.708 $\pm$ 0.030 & 0.821 $\pm$ 0.011 \\
12 & 0.738 $\pm$ 0.022 & 0.805 $\pm$ 0.012\\
16 & 0.734 $\pm$ 0.023 & 0.804 $\pm$ 0.012 \\
24 & 0.596 $\pm$ 0.091 & 0.809 $\pm$ 0.017 \\
32 & 0.682 $\pm$ 0.085 & 0.813 $\pm$ 0.011 \\
\bottomrule
\end{tabular}
\end{table}

\paragraph{GLUE8 benchmark.} 
We evaluate on eight GLUE tasks (SST-2, MRPC, RTE, QNLI, QQP, MNLI, CoLA, STS-B) using task metrics and report the average score across tasks.(Appendix~\ref{sec:supp-setup}). Figure~\ref{fig:glue8} shows that clustered and private adapters provide modest but consistent gains over naive sharing for budgets where routing is feasible (rank $\ge 8$ in our setup). Unlike GoEmotions, the improvements on GLUE are smaller in magnitude, which can be attributed to heterogeneous task formats and limited sharing opportunities once syntax and sentiment tasks are isolated by similarity. This asymmetry is consistent with the CR prediction: GoEmotions labels are subcategories of a shared emotion taxonomy with higher cross-label redundancy, whereas GLUE8 combines heterogeneous task formats with weaker cross-task structure, so clustering has less redundant signal to reuse. Nevertheless, the results show similarity-guided finetuning improves transfer under a parameter budget.

\begin{figure}
    \centering
    \includegraphics[width=0.7\linewidth]{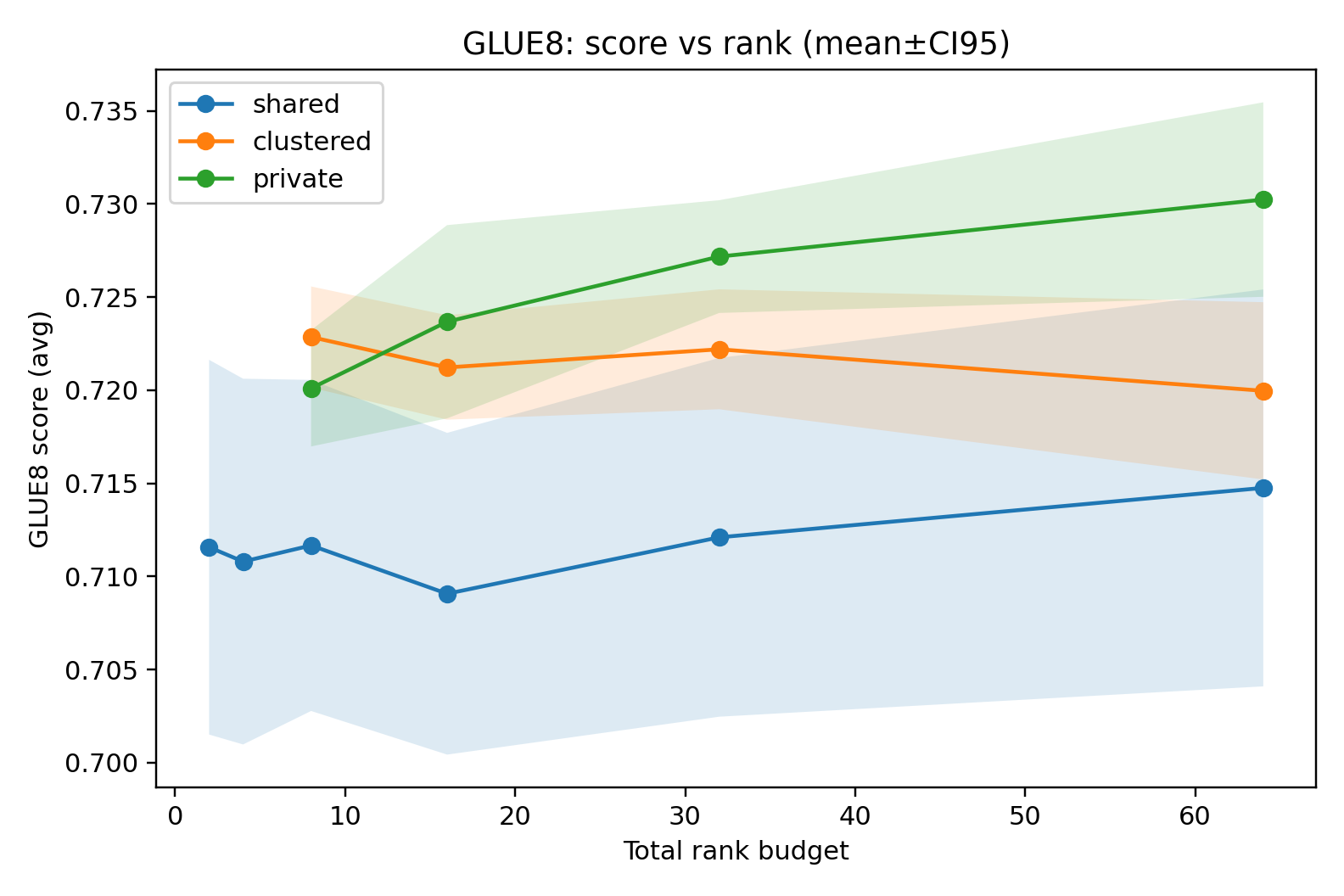}
  \caption{Shared is evaluated at all budgets, clustered/private are feasible at rank $\ge 8$ due to the number of adapters/tasks. We observe consistent improvement over naive sharing for the feasible budgets.}
  \label{fig:glue8}
\end{figure}

\begin{table}[t]
\centering
\caption{GLUE8: average score vs total rank budget (mean $\pm$ CI95) using official task metrics.}
\label{tab:glue8_main_curve}
\begin{tabular}{lrrrr}
\toprule
Rank & Regime & n & \shortstack{Score\\mean} & \shortstack{Score\\ci95} \\
\midrule
2 & shared & 10 & 0.712 & 0.010 \\
2 & clustered & 0 & -- & -- \\
2 & private & 0 & -- & -- \\
4 & shared & 10 & 0.711 & 0.010 \\
4 & clustered & 0 & -- & -- \\
4 & private & 0 & -- & -- \\
8 & shared & 10 & 0.712 & 0.009 \\
8 & clustered & 10 & 0.723 & 0.003 \\
8 & private & 10 & 0.720 & 0.003 \\
16 & shared & 10 & 0.709 & 0.009 \\
16 & clustered & 10 & 0.721 & 0.003 \\
16 & private & 10 & 0.724 & 0.005 \\
32 & shared & 10 & 0.712 & 0.010 \\
32 & clustered & 10 & 0.722 & 0.003 \\
32 & private & 10 & 0.727 & 0.003 \\
64 & shared & 7 & 0.715 & 0.011 \\
64 & clustered & 7 & 0.720 & 0.005 \\
64 & private & 6 & 0.730 & 0.005 \\
\bottomrule
\end{tabular}
\end{table}

Overall, our empirical results across GoEmotions and GLUE8 demonstrate that routing based on gradient similarity improves performance under fixed adapter capacity. So when tasks compete for limited shared capacity a naive sharing can amplify interference while specialization restores effective capacity aligned with our CR perspective.

On GoEmotions, clustered routing reduces $\widehat{\Delta}$ alongside improving AUROC (Figure~\ref{fig:goemotions_delta}). The Gaussian experiment clarifies how $\widehat{\Delta}$ acts like a regularization, increasing $N$ reduces variance, while $\varepsilon$ controls bias (Figure~\ref{fig:gaussian}). Thus, $\widehat{\Delta}$ can be used as a compact diagnostic of residual dependence in multi-task settings.

\section{Conclusion}
We introduced the CR inequality, a distribution-level converse that upper-bounds the total predictive information extractable from a shared latent by its information capacity plus the redundancy among tasks. This result explains negative transfer as an unavoidable consequence of low redundancy and finite capacity, and provides a principled approach for the design of shared--private architectures, including LoRA-based fine-tuning schemes.
Furthermore, we demonstrate a similarity-guided routing approach in PEFT improves performance on multitask under strict capacity budgets. Overall, our results suggest that under limited adapter capacity sharing parameters based on gradient based task similarity provides a practical basis for budget-aware specialization. Because TC measures redundancy, not synergy, CR does not capture complementary information across tasks. Extensions via partial information decomposition or multivariate interaction measures (e.g., $O$-information~\citep{rosas2019quantifying}) are possible future directions to go beyond redundancy. Moreover, extending the analysis to dynamic training settings may explain how optimization implicitly allocates capacity across tasks. 

\bibliography{uai}

@article{caruana1997multitask,
  title={Multitask learning},
  author={Caruana, Rich},
  journal={Machine Learning},
  volume={28},
  number={1},
  pages={41--75},
  year={1997},
  publisher={Springer}
}

@article{ruder2017overview,
  title={An overview of multi-task learning in deep neural networks},
  author={Ruder, Sebastian},
  journal={arXiv preprint arXiv:1706.05098},
  year={2017}
}

@inproceedings{zhang2014facial,
  title={Facial landmark detection by deep multi-task learning},
  author={Zhang, Zhanpeng and Luo, Ping and Loy, Chen Change and Tang, Xiaoou},
  booktitle={European Conference on Computer Vision},
  pages={94--108},
  year={2014},
  organization={Springer}
}

@inproceedings{misra2016cross,
  title={Cross-stitch networks for multi-task learning},
  author={Misra, Ishan and Shrivastava, Abhinav and Gupta, Abhinav and Hebert, Martial},
  booktitle={Proceedings of the IEEE conference on computer vision and pattern recognition},
  pages={3994--4003},
  year={2016}
}

@inproceedings{collobert2008unified,
  title={A unified architecture for natural language processing: Deep neural networks with multitask learning},
  author={Collobert, Ronan and Weston, Jason},
  booktitle={Proceedings of the 25th international conference on Machine learning},
  pages={160--167},
  year={2008}
}

@article{deng2013new,
  title={New types of deep neural network learning for speech recognition and related applications: An overview},
  author={Deng, Li and Hinton, Geoffrey and Kingsbury, Brian},
  journal={ICASSP},
  pages={8599--8603},
  year={2013}
}

@inproceedings{wang2019characterizing,
  title={Characterizing and avoiding negative transfer},
  author={Wang, Zirui and Dai, Zihang and P{\'o}czos, Barnab{\'a}s and Carbonell, Jaime},
  booktitle={Proceedings of the IEEE/CVF conference on computer vision and pattern recognition},
  pages={11293--11302},
  year={2019}
}

@inproceedings{houlsby2019parameter,
  title={Parameter-efficient transfer learning for NLP},
  author={Houlsby, Neil and Giurgiu, Andrei and Jastrzebski, Stanislaw and Morrone, Bryan and De Laroussilhe, Quentin and Gesmundo, Andrea and Attariyan, Mona and Gelly, Sylvain},
  booktitle={International Conference on Machine Learning},
  pages={2790--2799},
  year={2019}
}

@inproceedings{li2021prefix,
    title = "Prefix-Tuning: Optimizing Continuous Prompts for Generation",
    author = "Li, Xiang Lisa  and
      Liang, Percy",
    editor = "Zong, Chengqing  and
      Xia, Fei  and
      Li, Wenjie  and
      Navigli, Roberto",
    booktitle = "Proceedings of the 59th Annual Meeting of the Association for Computational Linguistics and the 11th International Joint Conference on Natural Language Processing (Volume 1: Long Papers)",
    month = aug,
    year = "2021",
    address = "Online",
    publisher = "Association for Computational Linguistics",
    url = "https://aclanthology.org/2021.acl-long.353/",
    doi = "10.18653/v1/2021.acl-long.353",
    pages = "4582--4597"
}

@article{ding2023parameter,
  title={Parameter-efficient fine-tuning of large-scale pre-trained language models},
  author={Ding, Ning and Qin, Yujia and Yang, Guang and Wei, Fuchao and Yang, Zonghan and Su, Yusheng and Hu, Shengding and Chen, Yulin and Chan, Chi-Min and Chen, Weize and others},
  journal={Nature machine intelligence},
  volume={5},
  number={3},
  pages={220--235},
  year={2023},
  publisher={Nature Publishing Group UK London}
}

@article{hu2022lora,
  title={Lora: Low-rank adaptation of large language models.},
  author={Hu, Edward J and Shen, Yelong and Wallis, Phillip and Allen-Zhu, Zeyuan and Li, Yuanzhi and Wang, Shean and Wang, Lu and Chen, Weizhu and others},
  journal={ICLR},
  volume={1},
  number={2},
  pages={3},
  year={2022}
}

@article{he2021towards,
  title={Towards a unified view of parameter-efficient transfer learning},
  author={He, Junxian and Zhou, Chunting and Ma, Xuezhe and Berg-Kirkpatrick, Taylor and Neubig, Graham},
  journal={arXiv preprint arXiv:2110.04366},
  year={2021}
}

@article{karimi2021compacter,
  title={Compacter: Efficient low-rank hypercomplex adapter layers},
  author={Karimi Mahabadi, Rabeeh and Henderson, James and Ruder, Sebastian},
  journal={Advances in neural information processing systems},
  volume={34},
  pages={1022--1035},
  year={2021}
}

@article{mahabadi2021parameter,
  title={Parameter-efficient multi-task fine-tuning for transformers via shared hypernetworks},
  author={Mahabadi, Rabeeh Karimi and Ruder, Sebastian and Dehghani, Mostafa and Henderson, James},
  journal={arXiv preprint arXiv:2106.04489},
  year={2021}
}

@article{yu2020gradient,
  title={Gradient surgery for multi-task learning},
  author={Yu, Tianhe and Kumar, Saurabh and Gupta, Abhishek and Levine, Sergey and Hausman, Karol and Finn, Chelsea},
  journal={Advances in neural information processing systems},
  volume={33},
  pages={5824--5836},
  year={2020}
}

@inproceedings{chen2018gradnorm,
  title={Gradnorm: Gradient normalization for adaptive loss balancing in deep multitask networks},
  author={Chen, Zhao and Badrinarayanan, Vijay and Lee, Chen-Yu and Rabinovich, Andrew},
  booktitle={International conference on machine learning},
  pages={794--803},
  year={2018},
  organization={PMLR}
}

@article{gururangan2021demix,
  title={Demix layers: Disentangling domains for modular language modeling},
  author={Gururangan, Suchin and Lewis, Mike and Holtzman, Ari and Smith, Noah A and Zettlemoyer, Luke},
  journal={arXiv preprint arXiv:2108.05036},
  year={2021}
}

@article{mudrakarta2018k,
  title={K for the price of 1: Parameter-efficient multi-task and transfer learning},
  author={Mudrakarta, Pramod Kaushik and Sandler, Mark and Zhmoginov, Andrey and Howard, Andrew},
  journal={arXiv preprint arXiv:1810.10703},
  year={2018}
}

@article{watanabe1960information,
  title={Information theoretical analysis of multivariate correlation},
  author={Watanabe, Satosi},
  journal={IBM Journal of research and development},
  volume={4},
  number={1},
  pages={66--82},
  year={1960},
  publisher={IBM}
}

@book{cover1999elements,
  title={Elements of information theory},
  author={Cover, Thomas M},
  year={1999},
  publisher={John Wiley \& Sons}
}

@article{sun1975linear,
  title={Linear dependence structure of the entropy space},
  author={Sun, TH},
  journal={Inf. Control},
  volume={29},
  number={4},
  pages={337--368},
  year={1975},
  publisher={Elsevier}
}

@article{madiman2010information,
  title={Information inequalities for joint distributions, with interpretations and applications},
  author={Madiman, Mokshay and Tetali, Prasad},
  journal={IEEE Transactions on Information Theory},
  volume={56},
  number={6},
  pages={2699--2713},
  year={2010},
  publisher={IEEE}
}

@inproceedings{evgeniou2004regularized,
  title={Regularized multi--task learning},
  author={Evgeniou, Theodoros and Pontil, Massimiliano},
  booktitle={Proceedings of the tenth ACM SIGKDD international conference on Knowledge discovery and data mining},
  pages={109--117},
  year={2004}
}

@article{ando2005framework,
  title={A framework for learning predictive structures from multiple tasks and unlabeled data.},
  author={Ando, Rie Kubota and Zhang, Tong and Bartlett, Peter},
  journal={Journal of machine learning research},
  volume={6},
  number={11},
  year={2005}
}

@article{argyriou2008convex,
  title={Convex multi-task feature learning},
  author={Argyriou, Andreas and Evgeniou, Theodoros and Pontil, Massimiliano},
  journal={Machine learning},
  volume={73},
  number={3},
  pages={243--272},
  year={2008},
  publisher={Springer}
}

@inproceedings{standley2020tasks,
  title={Which tasks should be learned together in multi-task learning?},
  author={Standley, Trevor and Zamir, Amir and Chen, Dawn and Guibas, Leonidas and Malik, Jitendra and Savarese, Silvio},
  booktitle={International conference on machine learning},
  pages={9120--9132},
  year={2020},
  organization={PMLR}
}

@inproceedings{zamir2018taskonomy,
  title={Taskonomy: Disentangling task transfer learning},
  author={Zamir, Amir R and Sax, Alexander and Shen, William and Guibas, Leonidas J and Malik, Jitendra and Savarese, Silvio},
  booktitle={Proceedings of the IEEE conference on computer vision and pattern recognition},
  pages={3712--3722},
  year={2018}
}

@article{baxter2000model,
  title={A model of inductive bias learning},
  author={Baxter, Jonathan},
  journal={Journal of artificial intelligence research},
  volume={12},
  pages={149--198},
  year={2000}
}

@article{maurer2016benefit,
  title={The benefit of multitask representation learning},
  author={Maurer, Andreas and Pontil, Massimiliano and Romera-Paredes, Bernardino},
  journal={Journal of Machine Learning Research},
  volume={17},
  number={81},
  pages={1--32},
  year={2016}
}

@article{tishby2000information,
  title={The information bottleneck method},
  author={Tishby, Naftali and Pereira, Fernando C and Bialek, William},
  journal={arXiv preprint physics/0004057},
  year={2000}
}

@article{alemi2016deep,
  title={Deep variational information bottleneck},
  author={Alemi, Alexander A and Fischer, Ian and Dillon, Joshua V and Murphy, Kevin},
  journal={arXiv preprint arXiv:1612.00410},
  year={2016}
}

@inproceedings{lin2020controllable,
  title     = {Controllable Pareto Multi-Task Learning},
  author    = {Lin, Xi and Zhen, Hui-Ling and Li, Zhenhua and Zhu, Qing-Fu and Kwok, James},
  booktitle = {Advances in Neural Information Processing Systems (NeurIPS)},
  volume    = {33},
  pages     = {20337--20348},
  year      = {2020},
  url       = {https://proceedings.neurips.cc}
}

@article{wang2020smalltowers,
  title={Small towers make big differences},
  author={Wang, Yuyan and Zhao, Zhe and Dai, Bo and Fifty, Christopher and Lin, Dong and Hong, Lichan and Chi, Ed H},
  journal={arXiv preprint arXiv:2008.05808},
  year={2020}
}

@misc{
newell2019featurepartition,
title={Feature Partitioning for Efficient Multi-Task Architectures},
author={Alejandro Newell and Lu Jiang and Chong Wang and Li-Jia Li and Jia Deng},
year={2020},
url={https://openreview.net/forum?id=B1eoyAVFwH}
}

@inproceedings{chai2023getmtl,
    title = "Improving Gradient Trade-offs between Tasks in Multi-task Text Classification",
    author = "Chai, Heyan  and
      Cui, Jinhao  and
      Wang, Ye  and
      Zhang, Min  and
      Fang, Binxing  and
      Liao, Qing",
    editor = "Rogers, Anna  and
      Boyd-Graber, Jordan  and
      Okazaki, Naoaki",
    booktitle = "Proceedings of the 61st Annual Meeting of the Association for Computational Linguistics (Volume 1: Long Papers)",
    month = jul,
    year = "2023",
    address = "Toronto, Canada",
    publisher = "Association for Computational Linguistics",
    url = "https://aclanthology.org/2023.acl-long.144/",
    doi = "10.18653/v1/2023.acl-long.144",
    pages = "2565--2579"
}

@article{qian2020mtvib,
  title={Multi-task variational information bottleneck},
  author={Qian, Weizhu and Chen, Bowei and Zhang, Yichao and Wen, Guanghui and Gechter, Franck},
  journal={arXiv preprint arXiv:2007.00339},
  year={2020}
}

@InProceedings{momma2022pareto,
  title = 	 {A Multi-objective / Multi-task Learning Framework Induced by Pareto Stationarity},
  author =       {Momma, Michinari and Dong, Chaosheng and Liu, Jia},
  booktitle = 	 {Proceedings of the 39th International Conference on Machine Learning},
  pages = 	 {15895--15907},
  year = 	 {2022},
  editor = 	 {Chaudhuri, Kamalika and Jegelka, Stefanie and Song, Le and Szepesvari, Csaba and Niu, Gang and Sabato, Sivan},
  volume = 	 {162},
  series = 	 {Proceedings of Machine Learning Research},
  month = 	 {17--23 Jul},
  publisher =    {PMLR},
  url = 	 {https://proceedings.mlr.press/v162/momma22a.html}
}

@article{rosas2019quantifying,
  title={Quantifying high-order interdependencies via multivariate extensions of the mutual information},
  author={Rosas, Fernando E and Mediano, Pedro AM and Gastpar, Michael and Jensen, Henrik J},
  journal={Physical Review E},
  volume={100},
  number={3},
  pages={032305},
  year={2019},
  publisher={APS}
}

@inproceedings{aghajanyan-etal-2021-intrinsic,
    title = "Intrinsic Dimensionality Explains the Effectiveness of Language Model Fine-Tuning",
    author = "Aghajanyan, Armen  and
      Gupta, Sonal  and
      Zettlemoyer, Luke",
    editor = "Zong, Chengqing  and
      Xia, Fei  and
      Li, Wenjie  and
      Navigli, Roberto",
    booktitle = "Proceedings of the 59th Annual Meeting of the Association for Computational Linguistics and the 11th International Joint Conference on Natural Language Processing (Volume 1: Long Papers)",
    month = aug,
    year = "2021",
    address = "Online",
    publisher = "Association for Computational Linguistics",
    url = "https://aclanthology.org/2021.acl-long.568/",
    doi = "10.18653/v1/2021.acl-long.568",
    pages = "7319--7328"
}

@article{fifty2021efficiently,
  title={Efficiently identifying task groupings for multi-task learning},
  author={Fifty, Chris and Amid, Ehsan and Zhao, Zhe and Yu, Tianhe and Anil, Rohan and Finn, Chelsea},
  journal={Advances in Neural Information Processing Systems},
  volume={34},
  pages={27503--27516},
  year={2021}
}

@inproceedings{demszky-etal-2020-goemotions,
    title = "{G}o{E}motions: A Dataset of Fine-Grained Emotions",
    author = "Demszky, Dorottya  and
      Movshovitz-Attias, Dana  and
      Ko, Jeongwoo  and
      Cowen, Alan  and
      Nemade, Gaurav  and
      Ravi, Sujith",
    editor = "Jurafsky, Dan  and
      Chai, Joyce  and
      Schluter, Natalie  and
      Tetreault, Joel",
    booktitle = "Proceedings of the 58th Annual Meeting of the Association for Computational Linguistics",
    month = jul,
    year = "2020",
    address = "Online",
    publisher = "Association for Computational Linguistics",
    url = "https://aclanthology.org/2020.acl-main.372/",
    doi = "10.18653/v1/2020.acl-main.372",
    pages = "4040--4054"
}

@inproceedings{kang2011learning,
  title={Learning with whom to share in multi-task feature learning},
  author={Kang, Zhuoli and Grauman, Kristen and Sha, Fei},
  booktitle={Proceedings of the 28th International Conference on Machine Learning},
  pages={521--528},
  year={2011}
}

@inproceedings{jacob2008clustered,
  title={Clustered multi-task learning: A convex formulation},
  author={Jacob, Laurent and Bach, Francis and Vert, Jean-Philippe},
  booktitle={Advances in Neural Information Processing Systems},
  volume={21},
  year={2008}
}

@inproceedings{denevi2022conditional,
  title={Conditional meta-learning of linear representations},
  author={Denevi, Giulia and Pontil, Massimiliano and Ciliberto, Carlo},
  booktitle={Advances in Neural Information Processing Systems},
  volume={35},
  year={2022}
}

@inproceedings{chua2021fine,
  title={How fine-tuning allows for effective meta-learning},
  author={Chua, Kurtland and Lei, Qi and Lee, Jason D.},
  booktitle={Advances in Neural Information Processing Systems},
  volume={34},
  year={2021}
}

@article{tian2025similar,
  title={Learning from similar linear representations: Adaptivity, minimaxity, and robustness},
  author={Tian, Ye and Gu, Yuqi and Feng, Yang},
  journal={Journal of Machine Learning Research},
  year={2025}
}

@article{hanneke2022no,
  title={A no-free-lunch theorem for multitask learning},
  author={Hanneke, Steve and Kpotufe, Samory},
  journal={The Annals of Statistics},
  volume={50},
  number={6},
  year={2022}
}
\appendix
\onecolumn
\section*{Supplementary Material}
\setcounter{section}{0}
\renewcommand{\thesection}{S\arabic{section}}
\renewcommand{\theequation}{S\arabic{section}.\arabic{equation}}
\renewcommand{\thetheorem}{S\arabic{section}.\arabic{theorem}}
\renewcommand{\thelemma}{S\arabic{section}.\arabic{lemma}}
\renewcommand{\thefigure}{S\arabic{section}.\arabic{figure}}
\renewcommand{\thetable}{S\arabic{section}.\arabic{table}}

\section{Notation and preliminaries}
\label{sec:supp_notation}
\begin{table}[h]
\centering
\begin{tabular}{@{}ll@{}}
\toprule
Symbol & Meaning \\
\midrule
$X$ & input variable \\
$Y_{1:T}=(Y_1,\dots,Y_T)$ & task labels \\
$Z$ & learned representation (shared or composite) \\
$Z_s$ & shared representation \\
$Z_t$ & private representation for task $t$ \\
$W$ & parameterized weight matrix  \\
$H(\cdot)$ & Shannon entropy \\
$I(A;B\mid C)$ & conditional mutual information \\
$\mathrm{TC}(Y_{1:T}\mid U)$ & conditional total correlation \\
$C_s, C_t$ & capacity budgets (information or proxy budget) \\
\bottomrule
\end{tabular}
\label{tab:notation}
\end{table}

\paragraph{Total correlation.}
\label{sec:supp_tc}
For label sequence $Y_{1:T}$, the TC and conditional TC terms are defined as,
\begin{equation}
\mathrm{TC}(Y_{1:T})
= \sum_{t=1}^T H(Y_t) - H(Y_{1:T}),
\qquad
\mathrm{TC}(Y_{1:T}\mid U)
= \sum_{t=1}^T H(Y_t\mid U) - H(Y_{1:T}\mid U).
\end{equation}

where $\mathrm{TC}(Y_{1:T}\mid U)\ge 0$ with equality iff $Y_1,\dots,Y_T$ are conditionally independent given $U$.

\paragraph{Identity underlying CR.}
For any representation $Z$ and side information $W$,
\begin{equation}
\sum_{t=1}^T I(Z;Y_t\mid W)
=
I(Z;Y_{1:T}\mid W)
+\mathrm{TC}(Y_{1:T}\mid W)
-\mathrm{TC}(Y_{1:T}\mid Z,W).
\label{eq:cr_identity}
\end{equation}

\begin{corollary}[Per-task error lower bound via Fano]\label{app:fano}
Assume each $Y^{(t)}$ takes values in a finite alphabet of size $|\mathcal{Y}_t|\ge 2$ and define the Bayes error given $Z_s$,
\(
P_{e,t}^\star = \inf_{\hat y_t(Z_s)} \mathbb{P}[\hat y_t(Z_s)\neq Y^{(t)}].
\)
Then
\begin{equation}
P_{e,t}^\star
\;\ge\;
\frac{H(Y^{(t)}\mid Z_s)-1}{\log |\mathcal{Y}_t|}
\;=\;
\frac{H(Y^{(t)}) - I(Z_s;Y^{(t)}) - 1}{\log |\mathcal{Y}_t|}.
\label{eq:fano_task}
\end{equation}
There exists a task $t^\star$ with
\begin{equation}
P_{e,t^\star}^\star
\;\gtrsim\;
\frac{H(Y^{(t^\star)}) - \frac{C_s+TC(Y^{(1:T)})}{T}}{\log |\mathcal{Y}_{t^\star}|}.
\label{eq:fano_pigeonhole}
\end{equation}
\end{corollary}

\section{Gaussian Specialization and Tightness}
\label{sec:gaussian}
A linear--Gaussian model offers an interpretable setting where we can evaluate the CR bounds exactly. Additionally, it shows exact conditions such as span alignment, orthogonality, and residual coupling under which these bounds become tight.

\textbf{Linear--Gaussian setup} Let $X\sim\mathcal{N}(0,\Sigma_X)\in\mathbb{R}^d$, we assume tasks are generated as noisy linear measurements,
\begin{equation}
Y^{(t)} \;=\; u_t^\top X + \varepsilon_t,
\qquad
\varepsilon_t\sim\mathcal{N}(0,\sigma_t^2),
\label{eq:gauss_tasks}
\end{equation}
where $\varepsilon_{1:T}\ \text{independent of }X$, we write $Y^{(1:T)} = U^\top X + \varepsilon$ where $U=[u_1,\dots,u_T]\in\mathbb{R}^{d\times T}$ and $\varepsilon\sim\mathcal{N}(0,\Sigma_\varepsilon)$ with $\Sigma_\varepsilon=\mathrm{diag}(\sigma_1^2,\dots,\sigma_T^2)$. A shared representation is produced by a linear map with an additive Gaussian noise:
\begin{equation}
Z_s \;=\; A X + \eta,
\qquad
\eta\sim\mathcal{N}(0,\sigma_z^2 I_m),
\qquad
\eta\perp (X,\varepsilon),
\label{eq:gauss_rep}
\end{equation}
where $A\in\mathbb{R}^{m\times d}$ is a design matrix, and $\sigma_z^2$ controls the bottleneck
noise. 

\subsection{Closed-form capacity and redundancy}
For the Gaussian channel \eqref{eq:gauss_rep},
\begin{equation}
I(Z_s;X)
=
\frac{1}{2}\log\det\!\left(I_m + \frac{1}{\sigma_z^2}A\Sigma_X A^\top\right).
\label{eq:gauss_capacity}
\end{equation}
The above expression makes the capacity dependence on the spectrum of $A\Sigma_X^{1/2}$ explicit.  At fixed output dimension $m$, capacity is governed by the nonzero singular values and saturates when $A\Sigma_X A^\top$ becomes low-rank or small-norm.

\textbf{Label redundancy.} Since $Y^{(1:T)}$ is jointly Gaussian with covariance \(
\Sigma_Y = U^\top \Sigma_X U + \Sigma_\varepsilon,
\)
the total correlation reduces to a log-determinant ratio,
\begin{equation}
TC(Y^{(1:T)})
=
\frac{1}{2}\log\frac{\prod_{t=1}^T \Sigma_{Y,tt}}{\det(\Sigma_Y)}.
\label{eq:gauss_tc}
\end{equation}
$TC(Y^{(1:T)})$ is large when task outputs are strongly correlated (e.g., when $\{u_t\}$ have large overlap in
the $\Sigma_X$-geometry), and vanishes when $\Sigma_Y$ is diagonal (independent tasks).

CR bound is tight in a Gaussian specialization under capacity saturation $I(Z_s;Y^{(1:T)})\approx I(Z_s;X)$, and explained dependence $\Delta=TC(Y^{(1:T)}\mid Z_s)\approx 0$ (labels become nearly conditionally independent given $Z_s$).

\textbf{When sharing is optimal.} In the Gaussian setting $I(Z_s;Y^{(1:T)})$ depends on how well the row space of $A$ captures the subspace spanned by $\{u_t\}$. Let the task subspace be $\mathcal{U} := \mathrm{span}\{u_1,\dots,u_T\}\subseteq\mathbb{R}^d$ in the $\Sigma_X$-geometry. If the row space of $A$ contains $\mathcal{U}$ (equivalently, $AX$ preserves all task-relevant directions), then $Z_s$ is a sufficient statistic for $U^\top X$ up to noise, and the joint information $I(Z_s;Y^{(1:T)})$ approaches $I(Z_s;X)$ as $\sigma_z^2\to 0$ under a fixed power constraint on $A$. Thus, conditioning on $Z_s$ removes most of the cross-task dependence induced by the shared latent $X$, so $\Delta\approx 0$ when the representation is informative enough about $U^\top X$. This results in a near-tightness of CR,
\begin{equation}
\sum_{t=1}^T I(Z_s;Y^{(t)})
\approx
I(Z_s;X) + TC(Y^{(1:T)}).
\label{eq:tight_aligned}
\end{equation}

\textbf{When sharing must fail at fixed rank.} Suppose task directions are nearly orthogonal in the $\Sigma_X$ inner product, so that $U^\top\Sigma_X U$ is close to diagonal. Then $TC(Y^{(1:T)})$ is small and there is little redundancy to reuse across tasks. Additionally, the representation rank is limited then $Z_s$ cannot preserve all task directions simultaneously. In this scenario, $I(Z_s;Y^{(t)})$ must be small for many tasks and the sum information is bounded by capacity alone,
\begin{equation}
\sum_{t=1}^T I(Z_s;Y^{(t)}) \;\lesssim\; I(Z_s;X) \;\le\; C_s,
\label{eq:orth_fail}
\end{equation}
$\text{when } TC(Y^{(1:T)})\approx 0$. This shows negative transfer under global sharing adding heterogeneous tasks forces interference unless the shared capacity scales with the number of distinct task directions.

\begin{proposition}[Near-tightness under span capture]\label{prop:gauss_tight}
Assume $\Sigma_X=I_d$ and $\sigma_t^2=\sigma^2$ for all $t$.
Let $\mathcal{U}=\mathrm{span}\{u_1,\dots,u_T\}$ have dimension $k$ and suppose $A$ has rank at least $k$ and row space
containing $\mathcal{U}$. Under a fixed power constraint $\mathrm{tr}(AA^\top)\le P$, there exists a choice of $A$ and
$\sigma_z^2$ such that
\begin{equation}
\sum_{t=1}^T I(Z_s;Y^{(t)})
\;\ge\;
I(Z_s;X) + TC(Y^{(1:T)}) - \epsilon.
\label{eq:gauss_tight_prop}
\end{equation}
\end{proposition}

\section{Full proofs}
\label{sec:supp_proofs}

\subsection{CR inequality and subset region}
\label{sec:supp_proof_cr}

\begin{theorem}[CR bound]
\label{thm:supp_cr}
Assume the Markov chain $Y_{1:T} - X - Z_s$ holds given $W$ and the fixed capacity constraint $I(X;Z_s\mid W)\le C_s$. Then
\begin{equation}
\sum_{t=1}^T I(Z_s;Y_t\mid W)\le C_s + \mathrm{TC}(Y_{1:T}\mid W).
\end{equation}
For any subset $S\subseteq[T]$ we can write,
\begin{equation}
\sum_{t\in S} I(Z_s;Y_t\mid W)\le C_s + \mathrm{TC}(Y_S\mid W).
\end{equation}
\end{theorem}

\begin{proof}
Using CR identity with $Z=Z_s$ and non-negativity of TC,
\[
\sum_{t=1}^T I(Z_s;Y_t\mid W)
\le I(Z_s;Y_{1:T}\mid W)+\mathrm{TC}(Y_{1:T}\mid W).
\]
By data-processing under $Y_{1:T}-X-Z_s$ given $W$,
$I(Z_s;Y_{1:T}\mid W)\le I(Z_s;X\mid W)\le C_s$.
The subset statement repeats the same argument for $Y_S$.
\end{proof}

\subsection{Shared--private extension}
\label{sec:supp_proof_sp}
\begin{theorem}[Shared--private CR bound]
\label{thm:supp_sp}
Assume $Y_{1:T}-X-(Z_s,Z_{1:T})$ is a Markov chain given $W$ and
$I(X;Z_s\mid W)\le C_s$, $I(X;Z_t\mid W)\le C_t$ for each $t$. Then,
\begin{equation}
\sum_{t=1}^T I((Z_s,Z_t);Y_t\mid W)
\le C_s+\sum_{t=1}^T C_t + \mathrm{TC}(Y_{1:T}\mid W).
\end{equation}
\end{theorem}

\begin{proof}
For each task $t$, $I((Z_s,Z_t);Y_t\mid W)\le I((Z_s,Z_{1:T});Y_t\mid W)$. Summing over $t$ and applying CR identity with $Z=(Z_s,Z_{1:T})$ gives,
\[
\sum_{t=1}^T I((Z_s,Z_{1:T});Y_t\mid W)
\le I((Z_s,Z_{1:T});Y_{1:T}\mid W)+\mathrm{TC}(Y_{1:T}\mid W).
\]
Using data-processing, $I((Z_s,Z_{1:T});Y_{1:T}\mid W)\le I((Z_s,Z_{1:T});X\mid W)$ with chain rule and non-negativity,
\[
I((Z_s,Z_{1:T});X\mid W)
\le I(Z_s;X\mid W)+\sum_{t=1}^T I(Z_t;X\mid W)
\le C_s+\sum_{t=1}^T C_t.
\]
\end{proof}

\subsection{LoRA capacity bound (Gaussian-noise channel)}
\label{sec:supp_lora_capacity}

\begin{theorem}[Linear-Gaussian channel upper bound]
\label{thm:supp_lora}
Let $U$ be any random vector with finite covariance $\Sigma_U=\mathrm{Cov}(U\mid W)$.
Let $Z = AU + \xi$ with $\xi\sim\mathcal{N}(0,\sigma^2 I)$ independent of $U$ given $W$.
Then
\begin{equation}
I(U;Z\mid W)\le \frac12\log\det\!\big(I + \sigma^{-2}A\Sigma_UA^\top\big).
\end{equation}
If additionally $\mathrm{rank}(A)\le r$ and $\mathrm{tr}(A\Sigma_UA^\top)\le \kappa$, then
\begin{equation}
I(U;Z\mid W)\le \frac{r}{2}\log\!\Big(1+\frac{\kappa}{r\sigma^2}\Big).
\end{equation}
\end{theorem}

\begin{proof}
Conditioning on $W$, we can write $\Sigma_Z=A\Sigma_UA^\top+\sigma^2 I$.
Since $Z|U$ is a Gaussian noise, $h(Z\mid U)=h(\xi)$.
By maximum-entropy $h(Z)\le \frac12\log((2\pi e)^m\det(\Sigma_Z))$.
Thus $I(U;Z)=h(Z)-h(Z\mid U)\le \frac12\log\frac{\det(\Sigma_Z)}{\det(\sigma^2I)}$. For the rank/trace bound, apply AM--GM on nonzero eigenvalues of $M=\sigma^{-2}A\Sigma_UA^\top\succeq 0$:
$\log\det(I+M)\le r\log(1+\mathrm{tr}(M)/r)$.
\end{proof}

\textbf{A Gaussian surrogate for LoRA capacity.} Consider a Gaussian/noisy-channel surrogate at inference time,
\begin{equation}
Z_s \;=\; A\,\phi(X) + \eta,
\qquad
\eta\sim\mathcal{N}(0,\sigma^2 I_r),
\label{eq:lora_noise}
\end{equation}
where $\eta$ is a noise used to make mutual information finite for deterministic encoders. Let $\Sigma_\phi := \mathrm{Cov}(\phi(X))$. Then, by the Gaussian channel formula,
\begin{equation}
I(Z_s;\phi(X))
=
\frac{1}{2}\log\det\!\left(I_r + \frac{1}{\sigma^2}A\,\Sigma_\phi\,A^\top\right).
\label{eq:lora_capacity_exact}
\end{equation}
where $\phi(X)$ is a deterministic function of $X$ given frozen weights. Now using DPI $I(Z_s;X) \;\ge\; I(Z_s;\phi(X))$ and $  I(Z_s;Y^{(1:T)}) \le I(Z_s;X),$, which implies \eqref{eq:lora_capacity_exact} is a conservative, architecture-linked proxy for the shared budget.

\textbf{Rank- and power-controlled upper bound.}
Assume a power constraint $\mathrm{tr}(AA^\top)\le \kappa$ and denote by $\lambda_1(\Sigma_\phi)\ge\dots\ge\lambda_d(\Sigma_\phi)$
the eigenvalues of $\Sigma_\phi$. By eigenvalue majorization and log-det bounds, $I(Z_s;\phi(X))
\;\le\;
\frac{r}{2}\log\!\Big(1+\frac{\kappa}{r\,\sigma^2}\,\lambda_1(\Sigma_\phi)\Big)
\;\le\;
\frac{1}{2}\sum_{i=1}^r \log\!\Big(1+\frac{\kappa}{\sigma^2}\,\lambda_i(\Sigma_\phi)\Big).$

This shows that effective shared capacity grows at most linearly with rank $r$, and it saturates if $\Sigma_\phi$ is low-rank or if $\kappa$ is small.

\section{Proof of Theorem~\ref{thm:clustering_gap}}
\label{app:clustering_gap_proof}
\begin{proof}[Proof of Theorem~\ref{thm:clustering_gap}]
Let $\mathcal{P}=\{S_1,\dots,S_K\}$ of $\tasks$ be a fixed partition. For any globally shared representation $Z_s$ and any collection of per-cluster representations $\{Z_{S_k}\}_{k=1}^K$,
we apply CR identity once with $(Z,Y^{(1:T)})=(Z_s,Y^{(1:T)})$ and once for each cluster with
$(Z,Y^{(1:|S_k|)})=(Z_{S_k},Y^{(S_k)})$,
\begin{align}
\sum_{t=1}^T I(Z_s;Y^{(t)})
&=
I(Z_s;Y^{(1:T)}) + \TC(Y^{(1:T)}) - \TC(Y^{(1:T)}\mid Z_s),
\label{eq:cr_global}\\
\sum_{t\in S_k} I(Z_{S_k};Y^{(t)})
&=
I(Z_{S_k};Y^{(S_k)}) + \TC(Y^{(S_k)}) - \TC(Y^{(S_k)}\mid Z_{S_k}).
\label{eq:cr_cluster_k}
\end{align}
Summing \eqref{eq:cr_cluster_k} over $k=1,\dots,K$ gives,
\begin{equation}
\label{eq:cr_cluster_sum}
\sum_{t=1}^T I(Z_{S_{k(t)}};Y^{(t)})
=
\sum_{k=1}^K I(Z_{S_k};Y^{(S_k)})
+
\sum_{k=1}^K \TC(Y^{(S_k)})
-
\sum_{k=1}^K \TC(Y^{(S_k)}\mid Z_{S_k}).
\end{equation}
Subtracting \eqref{eq:cr_global} from \eqref{eq:cr_cluster_sum} yields the \emph{general} gain identity
\begin{align}
\label{eq:gain_general}
\sum_{t=1}^T I(Z_{S_{k(t)}};Y^{(t)}) - \sum_{t=1}^T I(Z_s;Y^{(t)})
&=
\underbrace{\Big(\TC(Y^{(1:T)}\mid Z_s) - \sum_{k=1}^K \TC(Y^{(S_k)}\mid Z_{S_k})\Big)}_{\text{difference in residual couplings}}
\nonumber\\
&\quad
+\underbrace{\Big(\sum_{k=1}^K \TC(Y^{(S_k)}) - \TC(Y^{(1:T)})\Big)}_{\text{redundancy term}}
\nonumber\\
&\quad
+\underbrace{\Big(\sum_{k=1}^K I(Z_{S_k};Y^{(S_k)}) - I(Z_s;Y^{(1:T)})\Big)}_{\text{joint-predictive term}}.
\end{align}

Now note that by Definition~\ref{def:tc_between},
\[
\sum_{k=1}^K \TC(Y^{(S_k)}) - \TC(Y^{(1:T)}) = -\,\TCbetween(\mathcal{P}).
\]
Also define $\Dsh:=\TC(Y^{(1:T)}\mid Z_s)$ and $\Dcl{k}:=\TC(Y^{(S_k)}\mid Z_{S_k})$.
Plugging these into \eqref{eq:gain_general} gives
\begin{equation}
\label{eq:gain_with_joint_term}
\sum_{t=1}^T I(Z_{S_{k(t)}};Y^{(t)}) - \sum_{t=1}^T I(Z_s;Y^{(t)})
=
\Big(\Dsh - \sum_{k=1}^K \Dcl{k}\Big) \;-\; \TCbetween(\mathcal{P})
\;+\;
\Big(\sum_{k=1}^K I(Z_{S_k};Y^{(S_k)}) - I(Z_s;Y^{(1:T)})\Big).
\end{equation}

\cref{thm:clustering_gap} is stated for $Z_s^\star$ and $\{Z_{S_k}^\star\}$ that are capacity optimal under budgets $\sum_k C_k=C$. In the capacity constrained case the tightness of bound is when the joint-predictive term saturates its information budget,
\begin{equation}
\label{eq:saturation_assumption}
I(Z_s^\star;Y^{(1:T)}) = C
\quad\text{and}\quad
I(Z_{S_k}^\star;Y^{(S_k)}) = C_k\;\;\forall k,
\end{equation}
which holds in the Gaussian linear channel model when the encoder family can represent the optimal channel for each budget, and more generally whenever the optimum lies on the active constraint boundary. Under \eqref{eq:saturation_assumption},
\[
\sum_{k=1}^K I(Z_{S_k}^\star;Y^{(S_k)}) - I(Z_s^\star;Y^{(1:T)})
=
\sum_{k=1}^K C_k - C
=0.
\]
Therefore, \eqref{eq:gain_with_joint_term} reduces exactly to
\[
\sum_{t=1}^T I(Z_{S_{k(t)}}^\star;Y^{(t)}) - \sum_{t=1}^T I(Z_s^\star;Y^{(t)})
=
\Big(\Dsh - \sum_{k=1}^K \Dcl{k}\Big) \;-\; \TCbetween(\mathcal{P}),
\].
\end{proof}
The clustered setting is strictly better than global sharing iff the right-hand side is strictly positive $\Dsh - \sum_{k=1}^K \Dcl{k} \;>\; \TCbetween(\mathcal{P}).$  Without the saturation condition the exact gain identity consists of an additional joint-predictive difference term. In our empirical setup fixed-rank LoRA budget is typically budget active so that \cref{eq:saturation_assumption} is a good approximation, and the measured effect is dominated by the residual coupling and redundancy terms.
\section{Proof of Theorem~\ref{thm:grad_tc_bridge}}
\label{app:grad_tc_proof}

\begin{proof}[Proof of Theorem~\ref{thm:grad_tc_bridge}]
{
Using the squared-loss for a shared linear probe $\theta\in\mathbb{R}^d$, $\mathcal{L}_t(\theta)
=\frac12\,\mathbb{E}\big[(\theta^\top X - Y^{(t)})^2\big].$ Taking the derivative w.r.t $\theta$,
\[
\nabla_\theta \mathcal{L}_t(\theta)
=
\mathbb{E}\big[(\theta^\top X - Y^{(t)})X\big]
=
\Sigma\theta - \mathbb{E}[Y^{(t)}X].
\]
Under the model $Y^{(t)} = w_t^\top X + \varepsilon_t$ with $\varepsilon_t \perp X$ and $\mathbb{E}[\varepsilon_t]=0$,
\[
\mathbb{E}[Y^{(t)}X]
=
\mathbb{E}[(w_t^\top X)X]
=
\Sigma w_t.
\]
Hence, at initialization $\theta_0=0$,
\[
g_t \;=\; \nabla_\theta \mathcal{L}_t(\theta_0) = -\Sigma w_t.
\]
Therefore the raw population gradient cosine similarity is
\[
G^{\mathrm{raw}}_{ts}
=
\cos(g_t,g_s)
=
\frac{g_t^\top g_s}{\|g_t\|\,\|g_s\|}
=
\frac{w_t^\top \Sigma^2 w_s}{\|\Sigma w_t\|\,\|\Sigma w_s\|},
\]
which establishes the raw-gradient expression.

Let $\tilde X = \Sigma^{-1/2}X$, so $\tilde X\sim \mathcal{N}(0,I)$, and let us define $\tilde w_t = \Sigma^{1/2}w_t$. Then $Y^{(t)}=\tilde w_t^\top \tilde X + \varepsilon_t$ and
\[
\mathrm{Cov}(Y^{(t)},Y^{(s)})=\tilde w_t^\top \tilde w_s.
\]
Moreover, the natural gradient, equivalently the gradient at $\theta_0=0$ in the whitened coordinates, equals $-\tilde w_t$, so the whitened-gradient cosine is
\[
G^{\mathrm{nat}}_{ts}=\frac{\tilde w_t^\top \tilde w_s}{\|\tilde w_t\|\,\|\tilde w_s\|}.
\]
Since the denominators in the label correlation
\[
\rho_{ts}
=
\frac{\mathrm{Cov}(Y^{(t)},Y^{(s)})}{\sqrt{\mathrm{Var}(Y^{(t)})\,\mathrm{Var}(Y^{(s)})}}
\]
are strictly positive, $\mathrm{sign}(\rho_{ts})=\mathrm{sign}(\mathrm{Cov}(Y^{(t)},Y^{(s)}))=\mathrm{sign}(\tilde w_t^\top \tilde w_s)=\mathrm{sign}(G^{\mathrm{nat}}_{ts})$.

\textbf{Ordering under matched noise and normalized signal.} Assume matched noise $\sigma_t^2=\sigma^2$ for all $t$ and equal signal power
$\tilde w_t^\top \tilde w_t=\alpha^2$ for all $t$ (equivalently $w_t^\top \Sigma w_t$ constant across tasks). Then
\[
\mathrm{Var}(Y^{(t)})=\alpha^2+\sigma^2 \quad\text{for all } t,
\]
so the correlation simplifies to
\[
\rho_{ts}=\frac{\tilde w_t^\top \tilde w_s}{\alpha^2+\sigma^2}.
\]
Thus, for any two pairs $(t,s)$ and $(t',s')$,
\[
|\rho_{ts}| > |\rho_{t's'}|
\quad\Longleftrightarrow\quad
|\tilde w_t^\top \tilde w_s| > |\tilde w_{t'}^\top \tilde w_{s'}|.
\]
Under the same equal-norm condition, the cosine similarity satisfies
\[
|G^{\mathrm{nat}}_{ts}| = \frac{|\tilde w_t^\top \tilde w_s|}{\|\tilde w_t\|\,\|\tilde w_s\|}
= \frac{|\tilde w_t^\top \tilde w_s|}{\alpha^2}.
\]
Hence $|\rho_{ts}|$ and $|G^{\mathrm{nat}}_{ts}|$ induce the same ordering over pairs, proving (ii).
(If task signal powers differ, $G^{\mathrm{nat}}_{ts}$ still tracks the numerator but correlations are additionally rescaled by the marginal variances; the equal-power assumption isolates the redundancy structure.)

\textbf{Total correlation of Gaussian labels and clustering equivalence.} Let $Y=(Y^{(1)},\dots,Y^{(T)})$ be jointly Gaussian with covariance $\Sigma_Y$ and diagonal marginal variances $D=\mathrm{diag}(\sqrt{\Sigma_{Y,11}},\dots,\sqrt{\Sigma_{Y,TT}})$. Then $\Sigma_Y = D R D$, where $R$ is the correlation matrix. For a multivariate Gaussian,
\[
\TC(Y)=\sum_{t=1}^T H(Y^{(t)}) - H(Y)
=
\frac12\log\frac{\prod_{t=1}^T \Sigma_{Y,tt}}{\det(\Sigma_Y)}.
\]
Using $\det(\Sigma_Y)=\det(D)^2\det(R)=\big(\prod_t \Sigma_{Y,tt}\big)\det(R)$ yields
\[
\TC(Y) = -\frac12\log\det(R).
\]
Thus $R$ fully determines $\TC(Y)$. Under conditions in (ii) $G^{\mathrm{nat}}_{ts}$ is a strictly monotone rescaling of $\rho_{ts}$ (and hence of $R_{ts}$) on $[0,1]$. Similarly for absolute similarities when using $|\cdot|$. Therefore any clustering rule whose decisions depend only on the ordering of pairwise similarities produces the same partition when applied to $G^{\mathrm{nat}}$ as when applied to $R$. The raw-gradient cosine $G^{\mathrm{raw}}$ coincides with this quantity after whitening or when $\Sigma$ is isotropic; otherwise the difference is governed by the anisotropy of $\Sigma$.
This proves (iii).
}
\end{proof}

\section{Experimental setup and implementation details}
\label{sec:supp-setup}
\subsection{Datasets and evaluation metrics}
\textbf{GoEmotions.} We select the twelve most frequent labels in the training split and treat each as a binary task. For evaluation we compute AUROC per label on the validation split and report the macro-average across all tasks. We also report per-label AUC in~\cref{sec:supp-peremotion}.

\textbf{GLUE8.} We evaluate on SST-2, MRPC, RTE, QNLI, QQP, MNLI (matched), CoLA, and STS-B. We report standard metrics, accuracy for SST-2/RTE/QNLI/MNLI, $(\mathrm{Acc}+\mathrm{F1})/2$ for MRPC/QQP and Matthews correlation coefficient (MCC) for CoLA, and $(\rho_{\mathrm{Pearson}}+\rho_{\mathrm{Spearman}})/2$ for STS-B. We report the average of task scores as the GLUE8 score.

\subsection{Models and parameter-efficient fine-tuning}
All experiments use \texttt{bert-base-uncased} as the frozen encoder.
We attach LoRA adapters to the attention query/value projections and only LoRA parameters and task heads are trainable. For LoRA we use $\alpha=16$ with a dropout probability of $0.05$. 

We control parameterization through a total rank budget $r_{\mathrm{tot}}$.
In the shared setting, the single adapter has rank $r_{\mathrm{tot}}$. In clustered routing with $K$ adapters, we allocate ranks across adapters such that $\sum_{k=1}^K r_k = r_{\mathrm{tot}}$ with $r_k \ge 1$. For GLUE8 we use proportional rank allocation (larger clusters receive more rank), while GoEmotions uses fixed-sum allocation consistent with the split runner.

\textbf{Model training.} We train with AdamW, FP16 enabled, linear warmup ($6\%$ of total steps), and gradient clipping (norm $1.0$). We use batch size of $16$ and steps-per-epoch $1400$ for $3$ epochs. We use weight decay on task heads but set it to 0 for LoRA weights. This is done to avoid decoupled weight decay shrinking inactive adapters in multi-adapter setting.

\subsection{Similarity estimation and clustering}
We first partition tasks into $K$ clusters and use one adapter per cluster with ranks $\{r_k\}$ such that $\sum_k r_k=R$ and $r_k\ge 1$. We estimate task similarity without training full adapters by attaching a probe LoRA adapter of rank $1$, denoted $a_{\mathrm{probe}}$, and computing per-task gradient vectors,
\[
g_t := \frac{1}{B}\sum_{b=1}^B \nabla_{\theta_{\mathrm{probe}}}\,\ell_t(\theta_{\mathrm{probe}}; \text{batch }b),
\qquad
\tilde g_t := \frac{g_t}{\|g_t\|_2}.
\]
where for similarity we use cosine $s_{ij}=\langle \tilde g_i,\tilde g_j\rangle$. We use agglomerative clustering on the distance matrix $D_{ij}=1-s_{ij}$ with average linkage and fixed $K$.

For GoEmotions we use probe rank 1, 5 similarity seeds, and 20 gradient-accumulation steps per task, and for GLUE8 we use 30 steps per task. We warm-start task heads for 200 steps before computing gradients to stabilize the probe. Given the mean similarity matrix, we compute an average-linkage agglomerative clustering and select $K$ using a gap criterion constrained to $K \in [2,6]$. 

{
\subsection{Clustering-method sensitivity}
\label{sec:supp-clustering-sensitivity}
To test sensitivity to the clustering design, we sweep the number of clusters $K\in\{2,3,4,5,6\}$ and compare average-linkage, complete-linkage, and spectral clustering at ranks $8$ and $12$. Complete linkage gives the best AUC at the strongest setting ($K=6$, $r=8$), while average linkage tracks it closely and obtains the lowest residual-coupling value at $K=6$, $r=12$. Across algorithms, $\widehat{\Delta}$ generally decreases as $K$ grows, consistent with the prediction that more specialized routes reduce residual conflict.

\begin{table}[t]
\centering
\caption{Clustering sensitivity on GoEmotions: mean $\pm$ std over five random seeds.}
\label{tab:supp_clustering_sensitivity}
\begin{tabular}{lrrrr}
\toprule
Algorithm & K & Rank & avg\_auc & $\widehat{\Delta}$ \\
\midrule
avg & 2 & 8 & 0.6950 $\pm$ 0.0602 & 0.2331 $\pm$ 0.0286 \\
avg & 2 & 12 & 0.7011 $\pm$ 0.0729 & 0.2305 $\pm$ 0.0298 \\
avg & 3 & 8 & 0.7149 $\pm$ 0.0368 & 0.2229 $\pm$ 0.0145 \\
avg & 3 & 12 & 0.7021 $\pm$ 0.0559 & 0.2088 $\pm$ 0.0124 \\
avg & 4 & 8 & 0.8004 $\pm$ 0.0087 & 0.2341 $\pm$ 0.0089 \\
avg & 4 & 12 & 0.7817 $\pm$ 0.0174 & 0.2155 $\pm$ 0.0102 \\
avg & 5 & 8 & 0.8065 $\pm$ 0.0182 & 0.2290 $\pm$ 0.0093 \\
avg & 5 & 12 & 0.8081 $\pm$ 0.0228 & 0.2065 $\pm$ 0.0086 \\
avg & 6 & 8 & 0.8133 $\pm$ 0.0177 & 0.2074 $\pm$ 0.0091 \\
avg & 6 & 12 & 0.8101 $\pm$ 0.0208 & 0.2013 $\pm$ 0.0095 \\
complete & 2 & 8 & 0.7005 $\pm$ 0.0542 & 0.2340 $\pm$ 0.0271 \\
complete & 2 & 12 & 0.7092 $\pm$ 0.0575 & 0.2298 $\pm$ 0.0260 \\
complete & 3 & 8 & 0.8057 $\pm$ 0.0172 & 0.2437 $\pm$ 0.0102 \\
complete & 3 & 12 & 0.7945 $\pm$ 0.0180 & 0.2295 $\pm$ 0.0099 \\
complete & 4 & 8 & 0.7900 $\pm$ 0.0187 & 0.2322 $\pm$ 0.0103 \\
complete & 4 & 12 & 0.7882 $\pm$ 0.0141 & 0.2095 $\pm$ 0.0094 \\
complete & 5 & 8 & 0.8144 $\pm$ 0.0144 & 0.2272 $\pm$ 0.0095 \\
complete & 5 & 12 & 0.8067 $\pm$ 0.0208 & 0.2087 $\pm$ 0.0090 \\
complete & 6 & 8 & 0.8392 $\pm$ 0.0091 & 0.2159 $\pm$ 0.0083 \\
complete & 6 & 12 & 0.8355 $\pm$ 0.0117 & 0.2028 $\pm$ 0.0089 \\
spectral & 2 & 8 & 0.7598 $\pm$ 0.0296 & 0.2688 $\pm$ 0.0181 \\
spectral & 2 & 12 & 0.7632 $\pm$ 0.0271 & 0.2632 $\pm$ 0.0167 \\
spectral & 3 & 8 & 0.7799 $\pm$ 0.0186 & 0.2528 $\pm$ 0.0124 \\
spectral & 3 & 12 & 0.7546 $\pm$ 0.0474 & 0.2400 $\pm$ 0.0146 \\
spectral & 4 & 8 & 0.7824 $\pm$ 0.0139 & 0.2173 $\pm$ 0.0098 \\
spectral & 4 & 12 & 0.7899 $\pm$ 0.0076 & 0.2296 $\pm$ 0.0093 \\
spectral & 5 & 8 & 0.8124 $\pm$ 0.0119 & 0.2040 $\pm$ 0.0086 \\
spectral & 5 & 12 & 0.8057 $\pm$ 0.0094 & 0.2144 $\pm$ 0.0089 \\
spectral & 6 & 8 & 0.8033 $\pm$ 0.0083 & 0.2154 $\pm$ 0.0084 \\
spectral & 6 & 12 & 0.8008 $\pm$ 0.0127 & 0.2062 $\pm$ 0.0088 \\
\bottomrule
\end{tabular}
\end{table}
}

\subsection{Additional results}
\label{app:exp:results}
\section{Linear--Gaussian verification of the CR inequality}
\label{sec:supp-linear-gaussian}
We verify the CR inequality in a controlled linear--Gaussian setting where all quantities are computable in closed form. We sample $X\sim\mathcal{N}(0,I_d)$ ($d=20$) and define $T$ task labels
\begin{equation}
Y^{(t)} = a_t^\top X + \varepsilon_t,\qquad \varepsilon_t\sim\mathcal{N}(0,\sigma_\varepsilon^2),
\end{equation}
with independent $\{\varepsilon_t\}_{t=1}^T$ and $\sigma_\varepsilon=0.1$.
Let $A\in\mathbb{R}^{T\times d}$ stack $a_t^\top$ as rows. By varying the alignment of $\{a_t\}$ we sweep label redundancy. If $a_i \approx a_j$ then labels share signal and $\TC(Y^{1:T})$ increases, while near-orthogonality yields $\TC(Y^{1:T})\approx 0$.

Representations are produced by a rank-$r$ linear encoder
\begin{equation}
Z = W^\top X + \eta,\qquad \eta\sim\mathcal{N}(0,\sigma_\eta^2 I_r),
\end{equation}
with $\sigma_\eta=0.5$.
If $\mathrm{row}(W)\supseteq \mathrm{row}(A)$, then $A X$ is a linear function of $Z$, i.e.\ there exists $M$ such that
$AX = MZ$, and hence
\begin{equation}
Y = AX + \varepsilon = MZ + \varepsilon.
\end{equation}
Because $\varepsilon$ has independent coordinates, the tasks are conditionally independent given $Z$, implying $\TC(Y^{1:T}\mid Z)=0$, so the CR inequality becomes tight.

When the representation is rank-limited, the residual coupling term is exactly the conditional total correlation $\TC(Y^{1:T}\mid Z)$, which quantifies the remaining dependence among tasks after conditioning on the representation. We compute $\TC(Y^{1:T})$ and $\TC(Y^{1:T}\mid Z)$ from Gaussian covariance identities (log-determinant forms), and compute $I(Z;Y^{(t)})$ via conditional-variance formulas. Figure~\ref{fig:cr-experiments} reports sweeps across redundancy level, comparing optimal vs random vs misaligned encoders, number of tasks at low redundancy to illustrate negative transfer in average per-task information, and rank $r$ to show saturation once the encoder spans the task subspace.

\begin{figure}[t]
  \centering
  \begin{minipage}{0.32\textwidth}
    \centering
    \includegraphics[width=\linewidth]{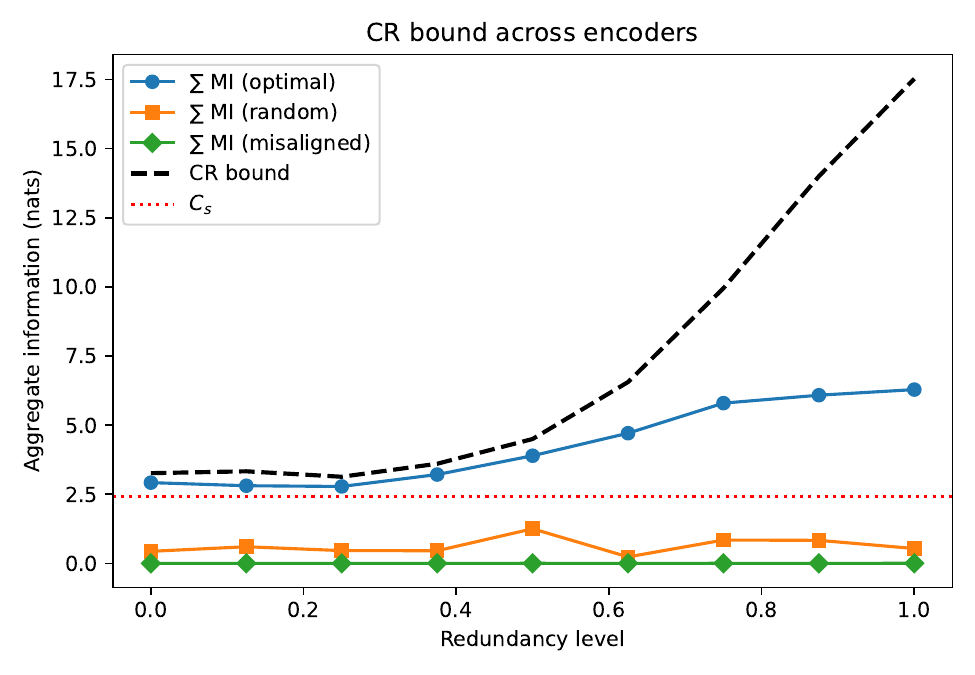}
  \end{minipage}\hfill
  \begin{minipage}{0.32\textwidth}
    \centering
    \includegraphics[width=\linewidth]{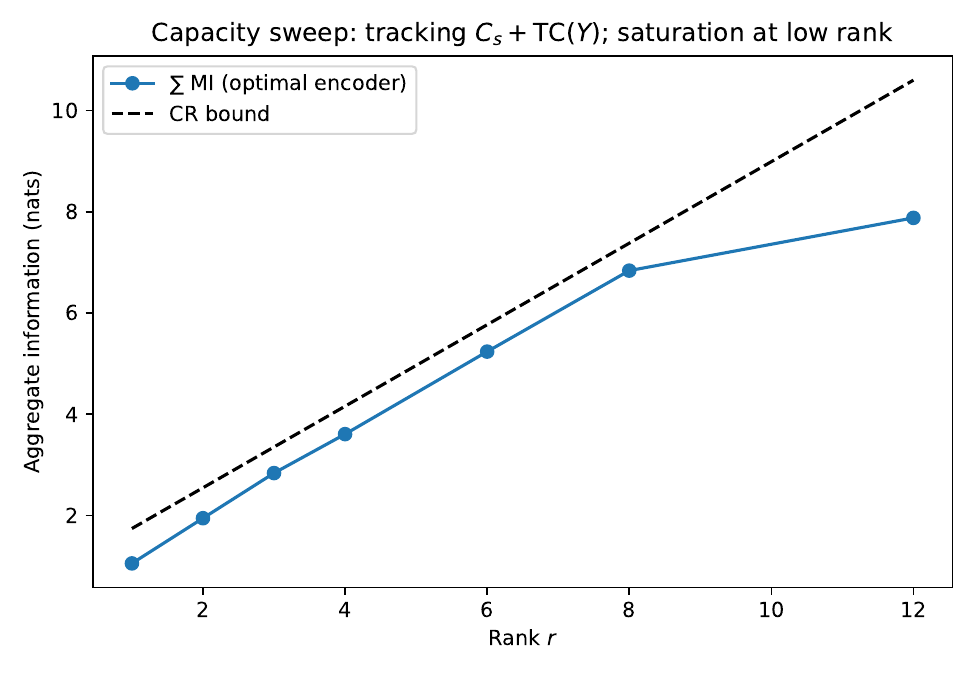}
  \end{minipage}\hfill
  \begin{minipage}{0.32\textwidth}
    \centering
    \includegraphics[width=\linewidth]{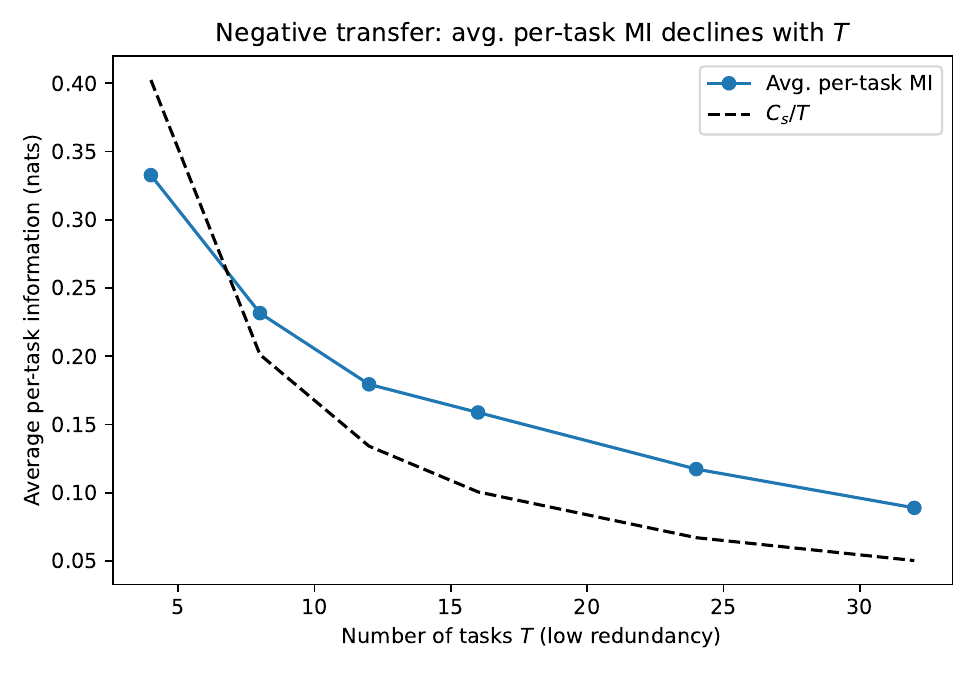}
  \end{minipage}
  \caption{Left: CR holds across encoder types, with tightness for the optimal encoder and slack for random or misaligned ones. Middle: Capacity sweep shows aggregate information tracks $C_s + \mathrm{TC}(Y)$, saturating at low rank. Right: Negative transfer demonstrated, as average per-task information declines when the number of tasks increases under fixed capacity and low redundancy.}
  \label{fig:cr-experiments}
\end{figure}

\subsection{GoEmotions: clustered vs random partitions (available baseline)}
\label{sec:supp-random}
\begin{figure}
    \centering
    \includegraphics[width=0.5\linewidth]{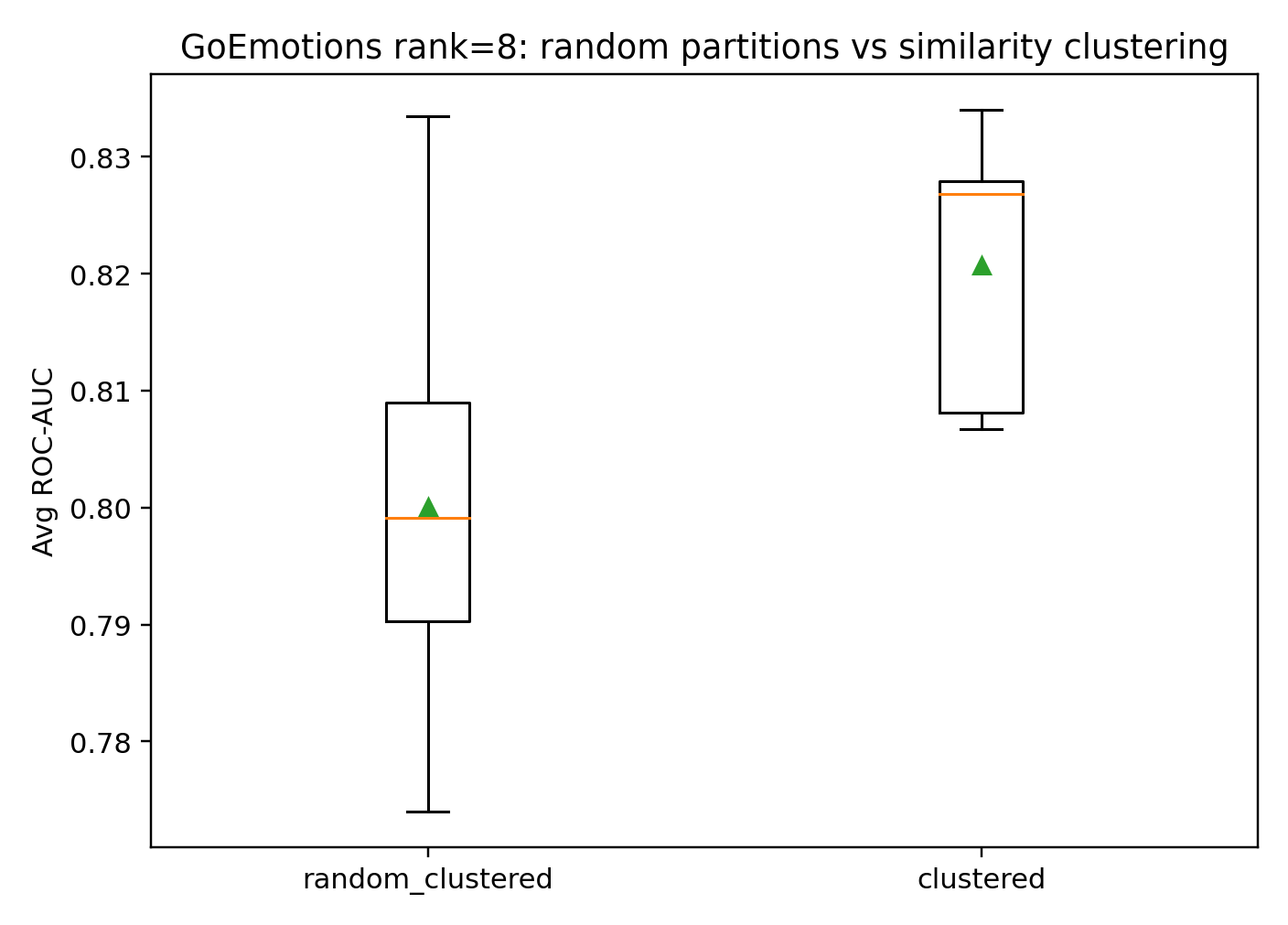}
  \caption{\textbf{GoEmotions (rank=8): clustered vs random partitions (available baseline).}
  Distribution of macro-average ROC-AUC over 45 random partitions (seed 41) compared to similarity-based clustering.
  We report this as a sanity check rather than a universal significance claim because random baselines are only available at rank 8 in the split-full runs.}
  \label{fig:supp_go_random}
\end{figure}

\begin{table}[t]
\centering
\caption{Clustered vs random-clustered (empirical p-value and effect size).}
\label{tab:goemotions_rank8_clustered_vs_random}
\begin{tabular}{lrrrrrr}
\toprule
Rank & Clustered\_mean & Random\_mean & p\_emp & Cohen\_d & N\_clustered & N\_random \\
\midrule
8.0000 & 0.821 & 0.800 & 0.0889 & 1.499 & 5 & 45 \\
\bottomrule
\end{tabular}
\end{table}

\subsection{Per-emotion gains at rank 12}
\label{sec:supp-peremotion}
\begin{figure}[t]
  \centering
  \includegraphics[width=\linewidth]{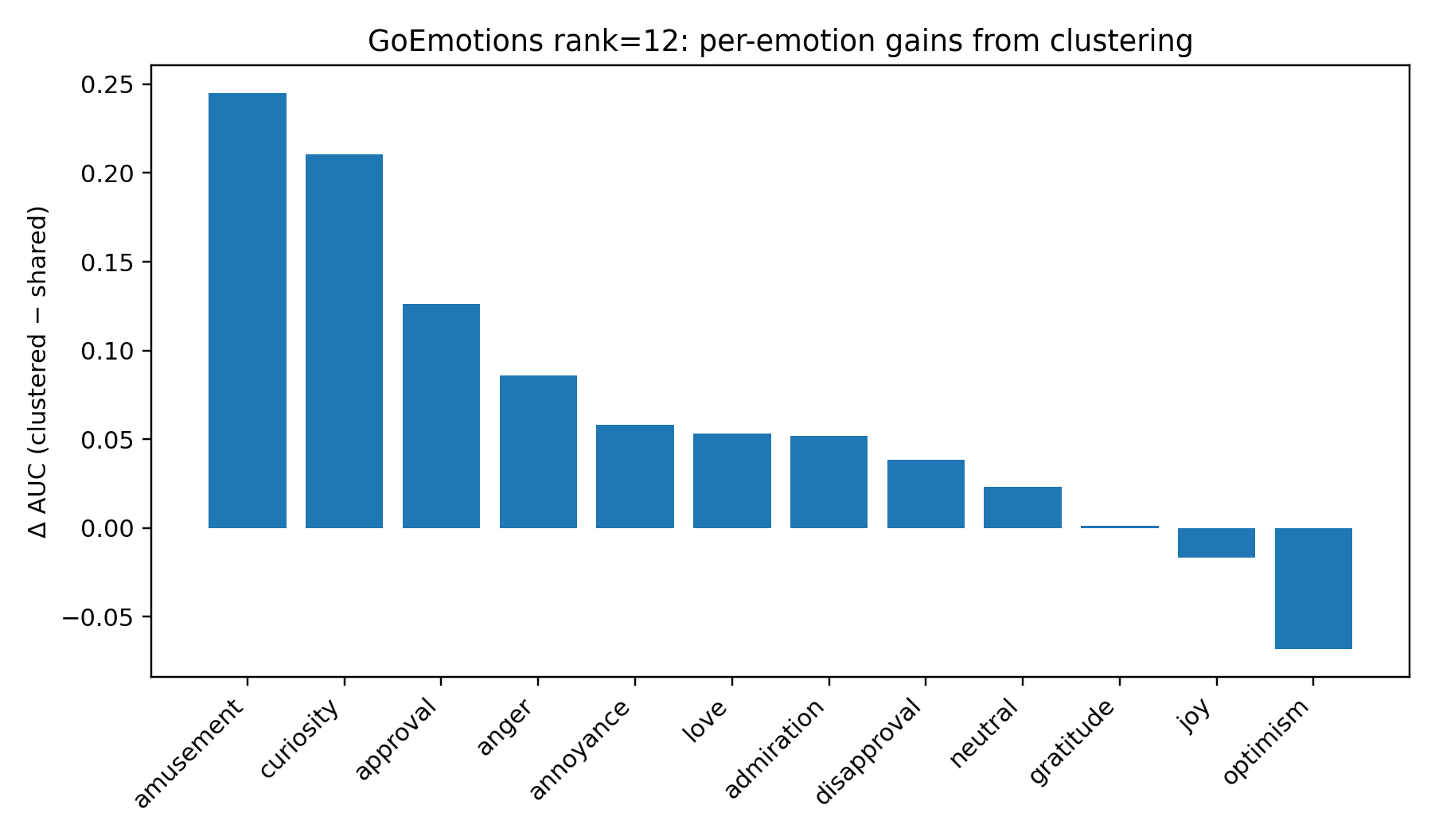}
  \caption{\textbf{GoEmotions per-label effect of clustering.} Difference in AUROC per emotion (clustered minus shared). Improvement in performance is on several harder labels (e.g., amusement, curiosity), while some labels change minimally. Error bars are reported in Table~\ref{tab:goemotions_rank12_per_emotion}.}
  \label{fig:supp_go_peremotion}
\end{figure}

\begin{table}[t]
\centering
\caption{Per-emotion AUROC for shared vs clustered (mean $\pm$ CI95) and their difference.}
\label{tab:goemotions_rank12_per_emotion}
\begin{tabular}{lrrrrr}
\toprule
Task & Shared\_mean & Shared\_ci95 & Clustered\_mean & Clustered\_ci95 & Delta\_mean \\
\midrule
amusement & 0.564 & 0.093 & 0.809 & 0.021 & 0.245 \\
curiosity & 0.675 & 0.094 & 0.886 & 0.017 & 0.211 \\
approval & 0.494 & 0.047 & 0.620 & 0.045 & 0.126 \\
anger & 0.753 & 0.074 & 0.839 & 0.033 & 0.086 \\
annoyance & 0.727 & 0.064 & 0.785 & 0.043 & 0.058 \\
love & 0.890 & 0.026 & 0.944 & 0.011 & 0.053 \\
admiration & 0.848 & 0.011 & 0.900 & 0.024 & 0.052 \\
disapproval & 0.704 & 0.047 & 0.742 & 0.040 & 0.038 \\
neutral & 0.761 & 0.012 & 0.784 & 0.009 & 0.023 \\
gratitude & 0.982 & 0.006 & 0.984 & 0.004 & 0.002 \\
joy & 0.800 & 0.008 & 0.784 & 0.055 & -0.016 \\
optimism & 0.653 & 0.041 & 0.585 & 0.136 & -0.068 \\
\bottomrule
\end{tabular}
\end{table}

\subsection{GoEmotions baselines}
\label{sec:supp-legacy}
We include additional baselines from the earlier harness using PCGrad, GradNorm, single-task oracle. Here, the single-task baseline is an oracle upper bound and is not compute-matched.

\begin{table}[t]
\centering
\caption{Additional baselines across available ranks (mean $\pm$ CI95).}
\label{tab:goemotions_legacy_baselines}
\begin{tabular}{lrrrr}
\toprule
Rank & Regime & N & Avg\_auc\_mean & Avg\_auc\_ci95 \\
\midrule
4 & shared & 5 & 0.725 & 0.032 \\
4 & clustered & 0 & -- & -- \\
4 & private & 0 & -- & -- \\
4 & single\_task & 0 & -- & -- \\
4 & pcgrad\_shared & 5 & 0.903 & 0.003 \\
4 & gradnorm\_shared & 5 & 0.894 & 0.001 \\
4 & random\_clustered & 0 & -- & -- \\
8 & shared & 5 & 0.702 & 0.081 \\
8 & clustered & 0 & -- & -- \\
8 & private & 0 & -- & -- \\
8 & single\_task & 0 & -- & -- \\
8 & pcgrad\_shared & 5 & 0.904 & 0.003 \\
8 & gradnorm\_shared & 5 & 0.901 & 0.002 \\
8 & random\_clustered & 0 & -- & -- \\
12 & shared & 5 & 0.738 & 0.012 \\
12 & clustered & 5 & 0.827 & 0.005 \\
12 & private & 5 & 0.830 & 0.003 \\
12 & single\_task & 5 & 0.929 & 0.000 \\
12 & pcgrad\_shared & 0 & -- & -- \\
12 & gradnorm\_shared & 0 & -- & -- \\
12 & random\_clustered & 55 & 0.824 & 0.003 \\
\bottomrule
\end{tabular}
\end{table}

\subsection{Similarity heatmaps}
\label{sec:supp-heatmaps}
\begin{figure}[t]
  \centering
  \includegraphics[width=\linewidth]{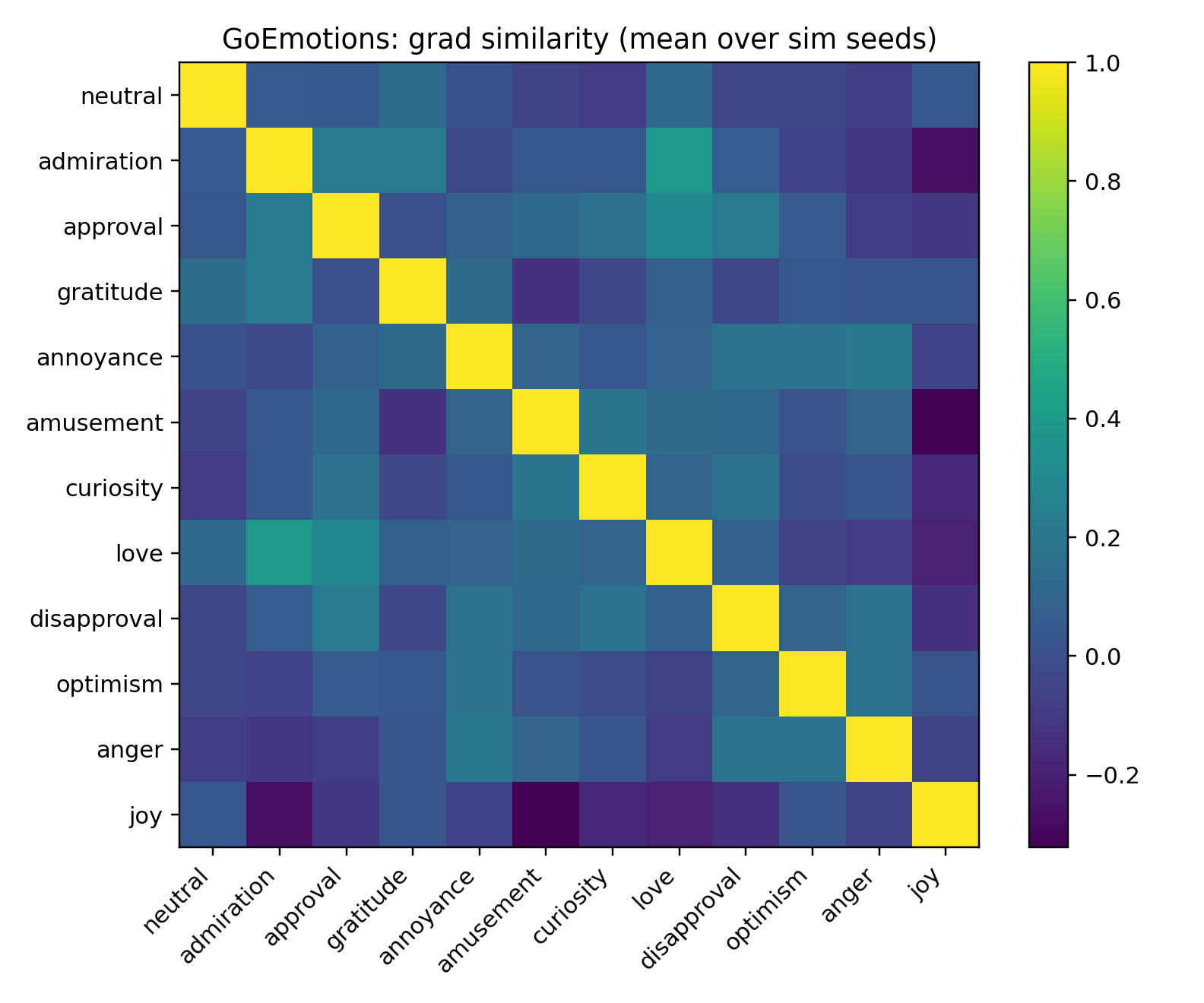}
  \caption{\textbf{Gradient similarity heatmap for GoEmotions.}
  Mean cosine similarity of probe gradients across five similarity seeds.
  The resulting clusters group semantically related labels.}
  \label{fig:supp_go_heatmap}
\end{figure}

\begin{figure}[t]
  \centering
  \includegraphics[width=\linewidth]{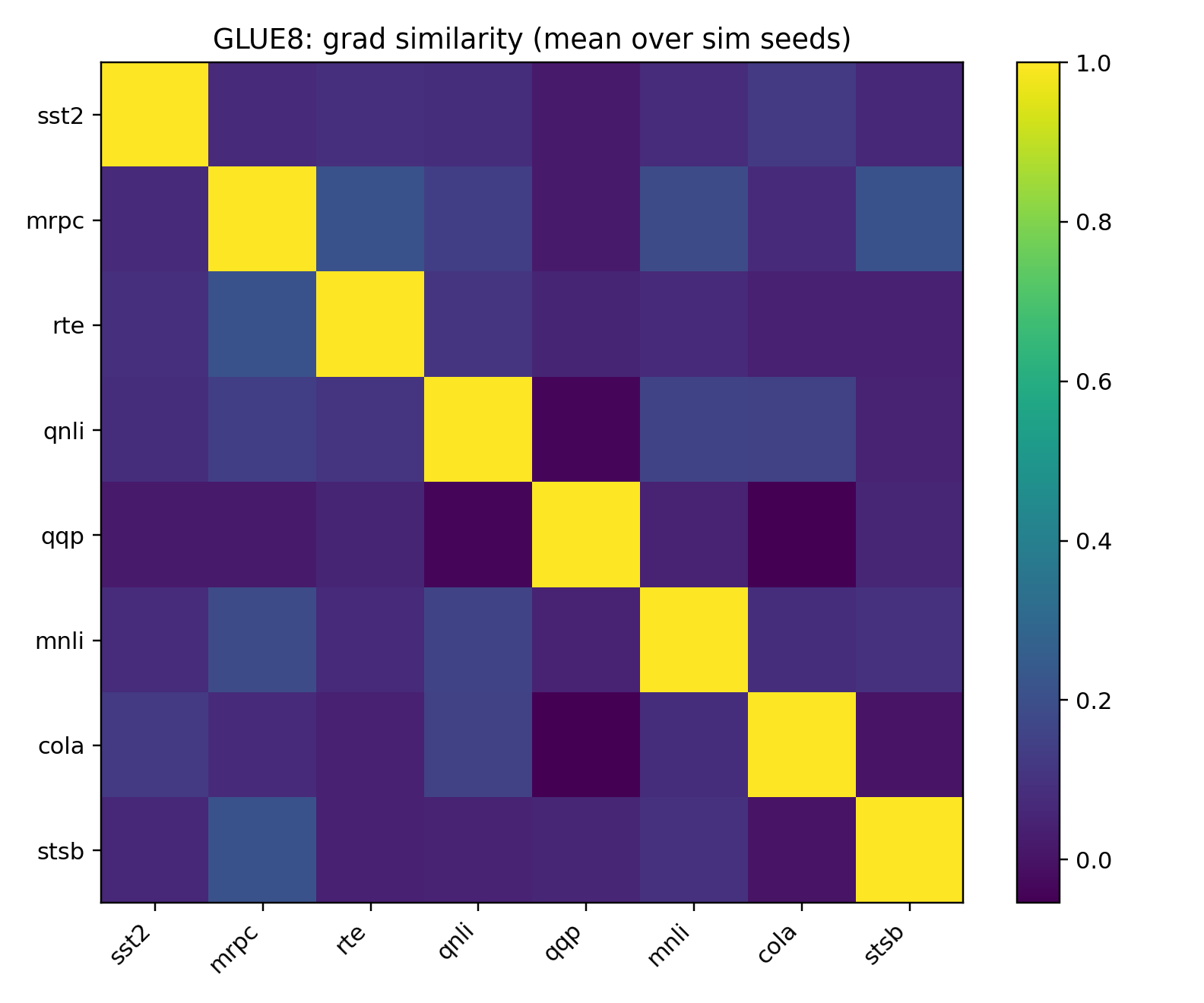}
  \caption{\textbf{Gradient similarity heatmap for GLUE}
  Mean cosine similarity of probe gradients across 5 similarity seeds.
  The clustering isolates sentiment (SST-2) and acceptability (CoLA) while grouping entailment/paraphrase-style tasks (e.g., MRPC--RTE and QNLI--MNLI).}
  \label{fig:supp_glue_heatmap}
\end{figure}

We make a note that clustered/private settings require $r_{\mathrm{tot}} \ge K$ (and private additionally requires $r_{\mathrm{tot}} \ge T$).For GLUE8 and GoEmotions, $K=6$ in our clustering, hence routing is only feasible for rank budgets $\ge 8$ in the plotted ranges.

\end{document}